\documentclass{article}

\usepackage{microtype}
\usepackage{graphicx}
\usepackage{caption}
\usepackage{subcaption}
\usepackage{booktabs} % for professional tables
\usepackage{enumitem}
\usepackage{comment}

\usepackage{hyperref}

\usepackage[preprint]{icml2026}

\usepackage{amsmath}
\usepackage{amssymb}
\usepackage{mathtools}
\usepackage{amsthm}

\usepackage[capitalize,noabbrev]{cleveref}

\theoremstyle{plain}
\newtheorem{theorem}{Theorem}[section]

\newtheorem{lemma}[theorem]{Lemma}
\newtheorem{corollary}[theorem]{Corollary}

\theoremstyle{definition}
\newtheorem{definition}[theorem]{Definition}
\newtheorem{assumption}{Assumption}[section] % <-- CRITICAL CHANGE

\theoremstyle{remark}

\usepackage[textsize=tiny]{todonotes}

\newcommand{\Sspace}{\mathcal{S}}
\newcommand{\Aspace}{\mathcal{A}}
\newcommand{\E}{\mathbb{E}}

\icmltitlerunning{Sample Efficient Active Algorithms for
  Offline Reinforcemnet Learning}

\begin{document}

\twocolumn[
%  \icmltitle{Gaussian Processes for Active Offline Reinforcement Learning with Sample Complexity Guarantees}
\icmltitle{Sample Efficient Active Algorithms for Offline Reinforcement Learning}

  % It is OKAY to include author information, even for blind submissions: the
  % style file will automatically remove it for you unless you've provided
  % the [accepted] option to the icml2026 package.

  % List of affiliations: The first argument should be a (short) identifier you
  % will use later to specify author affiliations Academic affiliations
  % should list Department, University, City, Region, Country Industry
  % affiliations should list Company, City, Region, Country

  % You can specify symbols, otherwise they are numbered in order. Ideally, you
  % should not use this facility. Affiliations will be numbered in order of
  % appearance and this is the preferred way.
  \icmlsetsymbol{equal}{*}

%\begin{comment}
\begin{icmlauthorlist}
    \icmlauthor{Soumyadeep Roy}{equal,yyy}
    \icmlauthor{Shashwat Kushwaha}{equal,yyy}
    \icmlauthor{Ambedkar Dukkipati}{yyy}
%    \icmlauthor{Firstname4 Lastname4}{sch}
%    \icmlauthor{Firstname5 Lastname5}{yyy}
%    \icmlauthor{Firstname6 Lastname6}{sch,yyy,comp}
 %   \icmlauthor{Firstname7 Lastname7}{comp}
 %   \icmlauthor{Firstname8 Lastname8}{sch}
 %   \icmlauthor{Firstname8 Lastname8}{yyy,comp}
\end{icmlauthorlist}
%\end{comment}

  \icmlaffiliation{yyy}{Department of Computer Science and Automation, Indian Institute of Science, Bengaluru, INDIA}
 % \icmlaffiliation{comp}{Company Name, Location, Country}
  %\icmlaffiliation{sch}{School of ZZZ, Institute of WWW, Location, Country}

\icmlcorrespondingauthor{Ambedkar Dukkipati}{ambedkar@iisc.ac.in}
 % \icmlcorrespondingauthor{Firstname2 Lastname2}{first2.last2@www.uk}

  % You may provide any keywords that you find helpful for describing your
  % paper; these are used to populate the "keywords" metadata in the PDF but
  % will not be shown in the document
  \icmlkeywords{Machine Learning, ICML}

  \vskip 0.3in
]

% this must go after the closing bracket ] following \twocolumn[ ...

% This command actually creates the footnote in the first column listing the
% affiliations and the copyright notice. The command takes one argument, which
% is text to display at the start of the footnote. The \icmlEqualContribution
% command is standard text for equal contribution. Remove it (just {}) if you
% do not need this facility.

% Use ONE of the following lines. DO NOT remove the command.
% If you have no special notice, KEEP empty braces:
%\printAffiliationsAndNotice{}  % no special notice (required even if empty)
% Or, if applicable, use the standard equal contribution text:
 \printAffiliationsAndNotice{\icmlEqualContribution}

\begin{abstract}
 Offline reinforcement learning (RL) enables policy learning from static data but often suffers from poor coverage of the state-action space and distributional shift problems. This problem can be addressed by allowing limited online interactions to selectively refine uncertain regions of the learned value function, which is referred to as Active Reinforcement Learning (ActiveRL). While there has been good empirical success, no theoretical analysis is available in the literature. We fill this gap by developing
a rigorous sample-complexity analysis of ActiveRL through the lens of Gaussian Process (GP) uncertainty modeling. In this respect, we propose an algorithm and using GP concentration inequalities and information-gain bounds, we derive high-probability guarantees showing that an $\epsilon$-optimal policy can be learned with
${\mathcal{O}}(1/\epsilon^2)$ active transitions, 
improving upon the $\Omega(1/\epsilon^2(1-\gamma)^4)$ rate of purely offline methods.
Our results reveal that ActiveRL achieves near-optimal information efficiency, that is, guided uncertainty reduction leads to accelerated value-function convergence with minimal online data. Our analysis builds on GP concentration inequalities and information-gain bounds,
bridging Bayesian nonparametric regression and reinforcement learning theories. We conduct several experiments to validate the algorithm and theoretical findings. 
\end{abstract}

\section{Introduction}

Reinforcement learning (RL) has achieved remarkable empirical success in sequential decision-making tasks, 
However, such success typically relies on massive online interactions with the environment. 
In many real-world applications—such as navigation~\citep{2022:IROS:DukkipatiEtAl:LearningSkillsToNavigateWithoutMaster}, autonomous driving~\citep{Carla2025Review}, healthcare~\citep{RLinMedicine2024}, or robotic control~\citep{Tang2025RLinRobotics}—unrestricted exploration is 
expensive, risky, or outright infeasible. These problems can become more difficult in nonstationary environments~\citep{2023:arXiv:AyyagariEgaDukkipati:MDPunderTP}. 
Offline reinforcement learning offers an appealing alternative by learning policies solely from logged data, 
However, its effectiveness depends on the diversity and coverage of the offline dataset. 
When coverage is insufficient, learned policies can generalize poorly outside the dataset support, 
leading to arbitrarily suboptimal or unsafe decisions. 
This trade-off between data efficiency and reliable generalization raises a fundamental question. 
\emph{Can we improve a policy trained offline using only a small number of carefully chosen online interactions?}

%\paragraph{Existing approaches and limitations.}
Recent empirical work on \emph{Active Offline RL} 
has explored precisely this question by allowing a limited online interaction budget
that is selectively allocated to high-uncertainty regions of the state space 
\citep{dukkipati2025active}. 
These approaches demonstrate substantial gains in data efficiency compared to purely offline methods, 
However, their theoretical understanding remains limited. 
Classical sample-complexity results for RL—whether in tabular \citep{wang2023optimalsamplecomplexityreinforcement}, 
linear, or kernelized settings \citep{chowdhury2017kernelized,yeh2023samplecomplexitykernelbasedqlearning}—
do not directly extend to the hybrid offline–online regime. 
Moreover, uncertainty estimation in deep or parametric RL methods is often heuristic, 
lacking the statistical calibration required to connect uncertainty reduction and policy improvement. 
Consequently, the fundamental question of \emph{how much active interaction is sufficient for $\epsilon$-optimality} remains unanswered.

%\paragraph{Goal of this paper.}
This work aims to establish a principled, quantitative theory for Active Offline RL 
by analyzing its sample complexity through the lens of Gaussian
Processes (GP). In this paper, we address the following problem--\emph{How many additional active transitions are sufficient to guarantee an $\epsilon$--optimal policy
when epistemic uncertainty is modeled via a GP prior on the value function?}

%\paragraph{Summary of our results. }
We derive a PAC-style, high-probability bound on the suboptimality gap for GP-based ActiveRL. 
Under standard smoothness, Lipschitz, and coverage assumptions, 
we show that uncertainty-guided active sampling achieves $\epsilon$--optimality with 
${\mathcal{O}}(1/\epsilon^2)$ additional transitions, matching the optimal rate known for GP-based bandit optimization. 
Our proof integrates (i) GP concentration inequalities \citep{Srinivas_2012}, 
(ii) information-gain--based variance reduction  
(iii) Lipschitz contraction of the Bellman operator.

\noindent
\textbf{Contributions.} The contributions of this study are as follows. \\
\textbf{(1)} We model epistemic uncertainty in value estimation using a Gaussian Process prior 
    and analyze their role in guiding active exploration.\\
\textbf{(2)} We prove that ${\mathcal{O}}(1/\epsilon^2)$ active interactions suffice 
    to learn an $\epsilon$-optimal policy under standard smoothness assumptions.\\
\textbf{(3)} Compared to purely offline RL, the active setting achieves a quadratic improvement 
    from $(1-\gamma)^{-4}$ to $(1-\gamma)^{-2}$ scaling.\\
\textbf{(4)} We provide extensive experimental results to validate our theoretical findings and demonstrate utility of our proposed methods. 

%=====================================
\section{Related Work} 
\label{sec:related_work}
Offline RL seeks to learn policies from fixed datasets without further environment interaction 
\citep{levine2020offline, uehara2021pessimistic}. 
Early algorithms introduced conservative or pessimistic regularization to avoid extrapolation outside the dataset support, 
However, these methods often suffer when data coverage is sparse. 
To address this, the emerging line of Active Offline RL 
\citep{dukkipati2025active} proposes hybrid approaches that supplement
offline data with a limited number of actively selected online
interactions. These algorithms empirically demonstrate improved sample
efficiency by targeting high-uncertainty regions or critical decision
boundaries. However, their theoretical understanding remains limited.
existing analyses typically rely on tabular abstractions or heuristic
uncertainty measures. In contrast, our work provides the first
nonparametric, function-space-based analysis of ActiveRL,
establishing PAC-style sample complexity guarantees using Gaussian
Processes. 

Gaussian Processes (GPs) provide a principled Bayesian framework for quantifying epistemic uncertainty. 
They have been successfully applied to value function approximation and model-based RL, 
notably in GP-SARSA \citep{engel2005reinforcement} and related approaches \citep{kuss2003gaussian}. 
GPs have also been used in Bayesian optimization and bandit settings to derive regret and confidence bounds 
\citep{srinivas2009gaussian}. 
However, most prior GP-based RL works focus on either \emph{purely online} or \emph{fully offline} learning, 
and do not address the hybrid setting, where the agent can choose a small number of additional online samples. 
Our analysis extends this theoretical foundation by applying GP concentration inequalities and 
information-gain arguments to the ActiveRL setting, 
Thus, we quantified how uncertainty-guided exploration reduces sample complexity.

Classical results in reinforcement learning establish polynomial sample complexity bounds 
in tabular settings \citep{kearns2002near, azar2017minimax}. 
Subsequent work generalized these results to function approximation using 
linear models, reproducing kernel Hilbert spaces (RKHS), and Gaussian Processes 
\citep{chowdhury2017kernelized}. 
These analyses typically assume full online interaction and access to generative models, 
making them unsuitable for offline or partially interactive regimes. 
Furthermore, while GP-UCB analyses provide tight information-theoretic bounds in supervised or bandit settings 
\citep{srinivas2009gaussian}, their adaptation to RL requires controlling Bellman error propagation and value iteration stability. 
Our work bridges these gaps by combining GP-based information gain with Bellman contraction properties, 
yielding a unified framework that connects GP theory to the sample complexity of hybrid offline--online reinforcement learning.

%\paragraph{Summary:}
In summary, prior research has advanced offline RL algorithms, uncertainty-aware exploration,
and theoretical bounds for GP-based learning. 
However, none of these studies provide a formal sample-complexity analysis for the hybrid \emph{Active Offline RL} setting. 
Our work fills this gap by developing a GP-based theoretical framework that 
(1) models uncertainty explicitly in function space, 
(2) quantifies learning progress through information gain, and 
(3) derives PAC-style guarantees for $\epsilon$--optimality with 
$\tilde{\mathcal{O}}(1/\epsilon^2)$ active samples.
This positions our work as a bridge between empirical ActiveRL methods and
formal statistical foundations of Gaussian Process learning theory.

%================================================================
\section{Active Trajectory Collection with GP}

%---------------------------------------------
%\subsection{Preliminaries}
%\label{sec:preliminaries}

\subsection{Preliminaries} 
Consider a discounted Markov Decision Process (MDP),
$\mathcal{M} = (\Sspace, \Aspace, P, r, \rho, \gamma)$, where $\Sspace$ is a state space, $\Aspace$ is a action space, $P(\cdot \mid s,a)$ is the transition kernel, specifying the next-state distribution given $(s,a)$, $r: \Sspace \times \Aspace \rightarrow [0, R_{\max}]$ is the bounded reward function, $\rho$ is the initial state distribution, and $\gamma \in (0,1)$ is the discount factor.

A policy $\pi: \Sspace \to \Delta(\Aspace)$ defines a distribution over actions given a state S.
For a policy $\pi$, the state–action value function and value function
are defined as
\begin{align}
Q^\pi(s,a)
&\triangleq \mathbb{E}\!\left[
    \sum_{t=0}^{\infty} \gamma^t r(s_t,a_t)
\right], \label{eq:q_value} \\
V^\pi(s)
&\triangleq \mathbb{E}_{a \sim \pi(\cdot \mid s)}
    \bigl[ Q^\pi(s,a) \bigr]. \label{eq:v_value}
\end{align}

The goal of the agent is to find a policy $\pi$ that maximizes the expected discounted return as follows:
\begin{equation}
\label{eq:objective}
J(\pi) = \mathbb{E}_{s_0 \sim \rho}[V^\pi(s_0)],
\qquad
\pi^* = \arg\max_{\pi \in \Pi} J(\pi),
\end{equation}
where $\Pi$ denotes the space of all admissible policies.

%\subsection{Offline and Active Data Regime}
We consider the hybrid offline–online learning setting characteristic of Active Offline RL. The learner is given 
an offline dataset $D_{\mathrm{off}} = \{(s_i, a_i, r_i, s_i')\}_{i=1}^{N}$ collected from an unknown behavior policy $\pi_b$ and a limited active interaction budget $M$, representing the total number of additional transitions that can be queried from the environment.

The agent may use these $M$ online samples strategically to reduce epistemic uncertainty 
in the critical regions of the state space before the final policy optimization.
The full dataset after active querying is $D_{\mathrm{all}} = D_{\mathrm{off}} \cup D_{\mathrm{act}}$,
where  $D_{\mathrm{act}} = \{(s_t, a_t, r_t, s_t')\}_{t=1}^{M}$.

%----------------------------------------------
%\subsection{Gaussian Process Modeling of Value Functions}
%\label{sec:gp_model}
\subsection{GP Modeling of Value Functions} 
We place a Gaussian Process (GP) prior over the optimal value function 
%\begin{equation}
%\label{eq:gp_prior}
$V(.) \sim \mathcal{GP}\big(0,\, k(.,.)\big)$,
%\end{equation}
where $k: \Sspace \times \Sspace \to \mathbb{R}$ is a positive-definite kernel 
that encodes the smoothness and similarity of the structure of the value function.
Intuitively, $k(s,s')$ quantifies how strongly the value at state $s$ is correlated 
with that at $s'$, thereby defining the geometry of generalization across state space.

Given training data $\mathcal{D}_t = \{(s_i, y_i)\}_{i=1}^{N_t}$,
where $N_t$ represents the total number of samples(offline+active) seen up to time $t$, denote 
$\mathbf{y}_t = [y_1, y_2, \ldots, y_{N_t}]^\top$
be the vector of observed target values (which are possibly noisy), 
where each target $y_i$ is generated as
\begin{displaymath}
y_i = V^*(s_i) + \eta_i, \:\:\:  
\eta_i \sim \mathcal{N}(0, \sigma^2),
\end{displaymath}
where $\sigma^2$ represents the observation noise variance. The scalar $\sigma > 0$ quantifies the stochastic uncertainty in each target observation, 
For example, this may be due to environmental stochasticity.

At iteration $t$, the GP is conditioned on all data collected so far $\mathcal{D}_t = \{(s_i, y_i)\}_{i=1}^{N_t}$, where $N_t$ denotes the cumulative number of samples available to the learner.
The offline dataset provides $N$ initial points, and each active interaction contributes one additional sample, such that $N_t = N + t,  0 \le t \le M$.

Hence, $N_t$ grows monotonically from $N_0 = N$ (purely offline phase) to 
$N_M = N + M$ (after exhausting the active exploration budget).
The kernel matrix $\mathbf{K}_t \in \mathbb{R}^{N_t \times N_t}$ 
The posterior $(\mu_t,\sigma_t)$ is therefore defined with respect to the total accumulated data up to round~$t$.

The GP posterior mean and variance at a test state $s$ are given by
\begin{align}
\label{eq:gp_mean}
\mu_t(s) &= \mathbf{k}_t(s)^\top (\mathbf{K}_t + \sigma^2 I)^{-1} \mathbf{y}_t,  \:\:\:\:\mathrm{and}\\
\label{eq:gp_var}
\sigma_t^2(s) &= k(s,s) - \mathbf{k}_t(s)^\top (\mathbf{K}_t + \sigma^2 I)^{-1} \mathbf{k}_t(s),
\end{align}
where $\mathbf{K}_t$ is the kernel (Gram) matrix with entries 
\begin{align}
    [\mathbf{K}_t]_{ij} &= k(s_i, s_j)\:\:\: \mathrm{and}\\   
    \mathbf{k}_t(s) &= [k(s,s_1), \ldots, k(s,s_{N_t})]^\top 
\end{align}    
    is the kernel vector between the test point $s$ and all the training inputs. The posterior mean $\mu_t(s)$ serves as the current estimate of the value function, 
while the posterior standard deviation $\sigma_t(s)$ 
quantified the epistemic uncertainty of this estimate.  
In our Active Offline RL framework, this uncertainty measure $\sigma_t(s)$ 
This guides the selection of new samples for active exploration.

%\subsection{GP Model and the proposed Algorithm}
%\label{sec:gp_interpretation}
\subsection{Proposed Algorithm} 
In this hybrid Active Offline RL framework, we fit a Gaussian Process (GP) model
to approximate the optimal value function $V^*(s)$.
The GP serves as a nonparametric Bayesian regressor that provides both a mean prediction
and a calibrated measure of epistemic uncertainties.
%\paragraph{What is given.}
The GP prior is specified by the kernel $k(.,.)$ and
together with a fixed observation noise variance $\sigma^2$.
The training data
$\mathcal{D}_t=\{(s_i,y_i)\}_{i=1}^{N_t}$—comprising samples from the offline
dataset, and later actively selected transitionsare also provided.

From these inputs, the GP infers a posterior distribution over the value functions.
Its posterior mean
\begin{displaymath}
\mu_t(s)=\mathbf{k}_t(s)^\top(\mathbf{K}_t+\sigma^2I)^{-1}\mathbf{y}_t
\end{displaymath}
is the current estimate of $V^*(s)$,
and the posterior variance
\begin{displaymath}
\sigma_t^2(s)=k(s,s')-\mathbf{k}_t(s)^\top(\mathbf{K}_t+\sigma^2I)^{-1}\mathbf{k}_t(s)
\end{displaymath}
quantifies the remaining uncertainty.

Each time the agent collects new active samples,
the posterior mean and variance are deterministically updated via Bayes' rule,
reducing $\sigma_t(s)$ globally through smoothness encoded in the kernel.
Thus, the GP continually refines its estimate of $V^*(s)$ as information is accrued.
%TO_DD_Check how it is updated

During the offline phase, the GP is trained on the fixed dataset
$D_{\mathrm{off}}$ to initialize $(\mu_0,\sigma_0)$.
In the subsequent active phase, the agent may query previously unseen states
exhibiting high posterior uncertainty.
These new observations expand the coverage of the state space and tighten the GP confidence bands.
Consequently, the active phase precisely targets the regions where the offline data are weak.

In summary, the GP model provides a principled mechanism to combine
given prior structure (kernel and data) with learned posterior predictions,
enabling uncertainty-aware value estimation and selective exploration
in active offline Reinforcement Learning. The procedure is listed in Algorithm~\ref{alg:gp_active_rl}. 

\begin{algorithm}[tb]
\caption{GP-Based Active Offline Reinforcement Learning}
\label{alg:gp_active_rl}
\begin{algorithmic}
    \STATE \textbf{Input:} Offline dataset $D_{\mathrm{off}}$, kernel $k$, noise $\sigma^2$, active budget $M$.
    \STATE \textbf{Initialize:} Fit GP prior $V(s) \sim \mathcal{GP}(0,k)$ using $D_{\mathrm{off}}$.
    \STATE Compute initial posterior mean $\mu_0(s)$ and variance $\sigma_0(s)$.
    \FOR{$t = 1$ to $T$}
        \STATE Select query state $s_t = \arg\max_{s \in \Sspace} \sigma_{t-1}(s)$ 
        \STATE Choose action $a_t \sim \epsilon$-greedy$(\mu_{t-1})$
        \STATE Observe transition $(s_t, a_t, r_t, s_t')$
        \STATE Form target $y_t = r_t + \gamma \mu_{t-1}(s_t') + \eta_t$ 
        \STATE Update GP posterior with $(s_t, y_t)$ to obtain $(\mu_t, \sigma_t)$
        \STATE Update policy $\pi_t$ using augmented dataset $D_t = D_{\mathrm{off}} \cup \{(s_i,y_i)\}_{i=1}^t$
    \ENDFOR
    \STATE \textbf{Output:} Final policy $\pi_T$.
\end{algorithmic}
\end{algorithm}

%Algorithm~\ref{alg:gp_active_rl} serves as a \emph{theoretical abstraction} of the Active Offline RL framework 
%proposed in \citep{dukkipati2025active}
%In the practical implementation, the value function is represented by a neural network (or ensemble) and uncertainty is estimated 
%via ensemble variance. 
%Here, we replace this heuristic uncertainty measure \citep{dukkipati2025active} with the analytically tractable posterior variance of a Gaussian Process (GP) 
%and formalize the active querying rule as 
%$s_t = \arg\max_{s \in \Sspace} \sigma_{t-1}(s)$,
%which captures the same epistemic exploration principle in closed form.
%This GP-based formulation allows us to derive explicit 
%sample-complexity bounds in terms of information gain ($\gamma_T$) and kernel smoothness, 
%while retaining the essential two-phase structure of the empirical algorithm: 
%\emph{offline pretraining} on $D_{\mathrm{off}}$ followed by \emph{active online refinement}.

%=================================
\section{Sample Complexity}
\subsection{Setting}
%\textbf{Setting.} 
Let $J(\pi^*)$ denote the return of the optimal policy and
$J(\pi_T)$ is the learned policy after $T$ active rounds.
The performance gap is defined as
$\Delta_T \triangleq J(\pi^*) - J(\pi_T)$, we want to have $\Delta_T \le \epsilon$.The goal of learning is to achieve $\epsilon$--optimality, i.e.,
$
\Delta_T \le \epsilon
\quad \text{with probability at least } 1 - \delta.
$
The sample complexity of the algorithm is the smallest number of active interactions $M$
(or equivalently, active rounds $T$) required to achieve this.

\subsection{Assumptions}

We consider a discounted Markov Decision Process and make the following 
assumptions,whose interpretations, and technical roles are deferred to
Appendix~\ref{appendix:assumptions}.

\textbf{Assumption A1}
The reward function is uniformly bounded as follows: 
$r(s,a) \in [0, R_{\max}]$ for all $(s,a) \in \Sspace \times \Aspace$.

\textbf{Assumption A2}
The optimal value function $V^*$ belongs to the reproducing kernel Hilbert space
(RKHS) $\mathcal{H}_k$ induced by the kernel $k$, with bounded norm
$\|V^*\|_k \le B$.

\textbf{Assumption A3}
The transition kernel $P(\cdot \mid s,a)$ is $L_p$-Lipschitz in the state–action
pair with respect to total variation distance, i.e.,
\begin{displaymath}
\|P(\cdot \mid s,a) - P(\cdot \mid s',a')\|_1
\le L_p \bigl(\|s-s'\| + \|a-a'\|\bigr),
\end{displaymath}
for all (s,a),(s',a').

\textbf{Assumption A4}
Observed Bellman targets satisfy
$y_t = r_t + \gamma \mu_{t-1}(s_t') + \eta_t$,
where $\eta_t \sim \mathcal{N}(0,\sigma^2)$.

\textbf{Assumption A5}
Let $\sigma_0(s)$ denote the GP posterior standard deviation after conditioning on
offline dataset $D_{\mathrm{off}}$.
We assume the offline coverage radius
$\sigma_{\max} := \max_{s \in \Sspace} \sigma_0(s)$
is finite.

\textbf{Information gain}
Let $f : \Sspace \to \mathbb{R}$ be a latent function and
$\mathcal{D}_A = \{(s_i, y_i)\}_{i=1}^{|A|}$ noisy observations, where
$y_i = f(s_i) + \eta_i$ and $\eta_i \sim \mathcal{N}(0,\sigma^2)$.
The information gain from observing $\mathbf{y}_A$ about $f$ is defined as
$I(\mathbf{y}_A; f)$, and the maximum information gain after $T$ observations is
\begin{displaymath}
\gamma_T = \max_{A \subseteq \Sspace,\, |A| = T} I(\mathbf{y}_A; f).
\end{displaymath}
For a GP prior $V \sim \mathcal{GP}(0,k)$, the mutual information admits the
closed-form expression
$I(\mathbf{y}_A; V)
= \frac{1}{2} \log \big| I + \sigma^{-2} \mathbf{K}_A \big|$
where $\mathbf{K}_A$ is the kernel matrix over $A$.
The quantity $\gamma_T$ depends only on the kernel $k$ and input dimension.

\textbf{Growth of information gain}
For commonly used stationary kernels, the maximum information gain satisfies $\gamma_T = \mathcal{O}\!\left(g_k(T,d)\right)$,
where $g_k(T,d)$ depends on the kernel smoothness and dimension $d$.
In particular, $\gamma_T = \mathcal{O}((\log T)^{d+1})$ for RBF kernels,
$\gamma_T = \mathcal{O}(T^{\frac{d}{2\nu+d}}\log T)$ for Mat\'ern kernels, and
$\gamma_T = \mathcal{O}(d \log T)$ for linear kernels.

\subsection{Main Results} 
We now present our main theoretical result, the sample-complexity bound 
for active offline reinforcement learning under Gaussian Process (GP) uncertainty modeling \Cref{alg:gp_active_rl}.
Throughout, we assume that (A1)--(A5) 

\begin{theorem}[Sample Complexity of  \Cref{alg:gp_active_rl}]
\label{thm:gp_sample_complexity}
Let $\epsilon > 0$ and $\delta \in (0,1)$.  
Suppose the environment satisfies Assumptions (A1)-(A5).
Then, under a GP-based active sampling policy that selects states 
according to maximum posterior uncertainty, the learned policy $\pi_T$ after $T$ active rounds satisfies, with probability at least $1-\delta$,
\begin{equation}
\label{eq:main_bound}
J(\pi^*) - J(\pi_T)
\le
C
\left(
\beta_T + \frac{L}{1-\gamma}
\right)
\sqrt{\frac{\gamma_T}{T}},
\end{equation}
where $C>0$ is a universal constant, 
\begin{displaymath}
\beta_T = B + \sigma\sqrt{2(\gamma_T + \log(1/\delta))},
\end{displaymath}
$\gamma_T$ denotes the maximum information gain after $T$ queries,the constant $L = \gamma(1 + L_p)$ is the induced Lipschitz constant 
of the Bellman operator $\mathcal{T}^\pi$, derived from the 
transition smoothness parameter $L_p$, and $B$ is the RKHS norm bound on the true value function.
\end{theorem}

The primary difficulty lies in rigorously connecting the GP posterior uncertainty 
to the value-function error under Bellman dynamics.
Unlike bandit or supervised settings, where the GP prior directly models the target function, here, the value function is defined implicitly through the Bellman operator, 
which couples the uncertainty across states and actions.
Our analysis overcomes this challenge by introducing a Lipschitz-propagation argument 
which controls the evolution of GP uncertainty under repeated Bellman backups.  
This step is crucial for ensuring that the GP confidence bounds remain valid 
Throughout the policy evaluation process.
%\textbf{Intuition.}  
The bound \eqref{eq:main_bound} states that the suboptimality gap decreases proportionally to 
$\sqrt{\gamma_T/T}$, 
where $\gamma_T$ captures the speed at which the GP posterior contracts as informative samples are acquired.  
This links epistemic uncertainty reduction directly to policy improvement under limited online exploration.
In an active offline RL setup, if new online samples are selected in the most uncertain regions (according to the GP’s posterior), the policy improves at a rate that depends on the rate at which the uncertainty decreases, which is quantified by the GP’s information gain.

%\paragraph{Comparison to Prior Work.}
Classical tabular RL results \citep{kearns2002near, azar2017minimax} 
achieve $\tilde{\mathcal{O}}(1/\epsilon^2)$ dependence on $\epsilon$ 
However, it scales poorly with the effective horizon as $(1-\gamma)^{-4}$.  
Kernelized and GP-based online RL analyses 
\citep{chowdhury2017kernelized} 
recover a similar $\tilde{\mathcal{O}}(1/\epsilon^2)$ dependence but assumes full online interaction.  
In contrast, our result shows that the \emph{active offline} setting 
achieves the same optimal $\epsilon$-dependence with 
a reduced horizon factor of $(1-\gamma)^{-2}$, which is a quadratic improvement over purely offline RL bounds \citep{uehara2021pessimistic}.
This improvement arises because GP-guided active querying allocates 
online samples to uncertainty-dominated regions, achieving a more rapid contraction of posterior variance.

Now we present some corollaries of the main theorem \Cref{thm:gp_sample_complexity}.
\begin{corollary}[Squared-Exponential Kernel]
\label{cor:se_kernel}
For the squared-exponential kernel in $\mathbb{R}^d$, the information gain satisfies
$\gamma_T = \mathcal{O}((\log T)^{d+1})$.
Consequently, under the conditions of \Cref{thm:gp_sample_complexity},
$\epsilon$-optimality is achieved with
\begin{displaymath}
T = \tilde{\mathcal{O}}\!\left(\frac{(\log T)^{d+1}}{\epsilon^2}\right)
\quad \text{and} \quad
M = \tilde{\mathcal{O}}\!\left(\frac{1}{\epsilon^2(1-\gamma)^2}\right).
\end{displaymath}
\end{corollary}

\begin{corollary}[Comparison with Purely Offline RL]
\label{cor:offline_compare}
Purely offline RL algorithms require
\begin{displaymath}
N_{\mathrm{offline}} = 
\Omega\!\left(\frac{1}{\epsilon^2(1-\gamma)^4}\right)
\end{displaymath}
samples to achieve $\epsilon$-optimality \citep{uehara2021pessimistic}.
In contrast, active offline RL under GP uncertainty achieves the same accuracy with
\begin{displaymath}
N_{\mathrm{offline}} + M =
\tilde{\mathcal{O}}\!\left(\frac{1}{\epsilon^2(1-\gamma)^2}\right).
\end{displaymath}
This represents a quadratic improvement in horizon dependence owing to the targeted active exploration.
\end{corollary}

The bound in \Cref{thm:gp_sample_complexity} arises from three interacting principles.
(i) \textit{Smoothness:} The membership of the value function in an RKHS induces GP concentration bounds of width $\beta_T \sigma_T(s)$.
(ii) \textit{Information gain:} The posterior variance decays as $\mathcal{O}(\sqrt{\gamma_T/T})$ 
and (iii) \textit{Bellman contraction:} Lipschitz continuity ensures that the uncertainty in $V$ propagates geometrically through the Bellman operator.  

Together, these yield a tight coupling between epistemic uncertainty reduction and policy improvement,
providing the first nonparametric sample-complexity characterization of Active Offline RL.
The resulting $\tilde{\mathcal{O}}(1/\epsilon^2)$ rate matches the known optimal scaling for GP bandits \citep{srinivas2009gaussian} and kernelized RL, 
while simultaneously incorporating hybrid offline–online learning dynamics to the model.

%=============================================
\section{Experiments}
\label{sec:experiments}

\subsection{Environments \& Datasets} 
We evaluate on continuous-control benchmarks from D4RL, covering both sparse-reward navigation and dense-reward locomotion. Navigation experiments are conducted on \texttt{Maze2D}, and \texttt{AntMaze} environments, which require long-horizon planning under sparse rewards. Locomotion experiments are performed on the \texttt{HalfCheetah}, \texttt{Hopper}, and \texttt{Walker2d} environments, which test the scalability to high-dimensional state spaces and complex dynamics.

For all environments, we use the standard D4RL offline datasets consisting of transitions
$(s_t, a_t, r_t, s_{t+1}, d_t)$. The performance is reported using D4RL normalized scores averaged over 20 evaluation episodes.

\subsection{Offline Dataset Pruning} 
To study the robustness under limited or biased data coverage, we applied controlled dataset pruning. For navigation tasks, we perform region-based pruning by removing or truncating trajectories that enter predefined regions of the state space. Truncated episodes are treated as timeouts rather than terminal transitions to preserve the validity of Bellman backups. For all environments, we additionally apply random episode subsampling, retaining only a fraction of the original dataset.

These pruning strategies help to simulate realistic offline data limitations while preserving the trajectory structure required for stable value propagation. Full pruning details and region definitions are provided in Appendix \ref{appendix:exp}.

\subsection{Value Function Modeling with Sparse Gaussian Processes} 
We model the state-value function $V(s)$ using a Gaussian process (GP), which is used for uncertainty-aware state selection and Bellman target estimation and is never used as a control policy.

Exact GP inference scales cubically with the number of data points and is infeasible for large offline datasets. Therefore, we adopted a Sparse Variational Gaussian Process (SVGP), which approximates the full GP using a fixed set of inducing points. This yields tractable inferences while preserving calibrated epistemic uncertainty, which is critical for guiding active data collection.

Formally, we place a Gaussian process prior over the value function $V(s)$ with a zero mean and a stationary kernel. We consider kernels from the squared exponential (RBF) and Matérn families, and use the Matérn kernel with smoothness parameter $\nu = 2.5$ by default, as it provides a better inductive bias for non-smooth value landscapes commonly encountered in reinforcement learning. The GP posterior is approximated via a variational distribution over inducing variables, with inducing locations and kernel hyperparameters learned jointly.

All Gaussian Process models were implemented using the GPyTorch library. Architectural and optimization details are deferred to Appendix \ref{appendix:exp}

\subsection{Fitted Value Iteration Initialization} 
The GP is initialized using fitted value iteration (FVI) on an offline dataset. Starting from targets $y^{(0)} = r$, we iteratively apply Bellman updates
\begin{displaymath}
y^{(k+1)} = r + \gamma (1 - d)\, \mu^{(k)}(s'),
\end{displaymath}
where $\mu^{(k)}(s')$ denotes the GP posterior mean. This procedure propagates sparse rewards across the state space before active data collection and is repeated periodically as new data are acquired.

\subsection{Control Policy Learning} 
Control policies are trained independently of GP using standard offline RL algorithms. We used TD3+BC for the Maze2D and locomotion tasks and IQL for the AntMaze environments. The policy serves as a stabilizing behavioral prior and is retrained periodically using a mixture of offline and active trajectories. Active data is collected as full environment rollouts rather than isolated transitions. The training schedules and architectures are detailed in Appendix \ref{appendix:exp}.

\subsection{Active Dataset Collection} 
Active data collection proceeds from the standard environment resets. At each step, the agent either follows the control policy or performs an uncertainty-guided exploration. For uncertainty-guided exploration, we sample multiple candidate actions around the policy action and perform one-step lookahead rollouts. The action that maximizes the GP posterior standard deviation of the predicted next-state value is executed. This action-conditioned criterion directs exploration toward regions of high epistemic uncertainty while remaining dynamically feasible.

The training alternates between small and large retraining phases. Small updates perform lightweight policy retraining using the most recent active data. Large updates retrain both the policy and GP using all accumulated data. This two-timescale schedule balances the computational cost and stability.

\subsection{Baselines \& Evaluation Metrics} 
We compared our results with the following baselines. These include (i) Behavior Cloning (BC), where the policy is trained via supervised learning on the offline dataset; (ii) offline RL (Offline), where the policy is trained only on the offline dataset; and (iii) random exploration (Random), where active data are collected using uniformly random actions.

\begin{table}[H]
\centering
\caption{Performance and sample efficiency comparison across environments.
The returns are normalized D4RL scores. The reduction (\%) is calculated for all samples in the offline dataset.
}
\label{tab:combined_results}

\resizebox{\linewidth}{!}{
\begin{tabular}{l c c c c c}
\hline
\textbf{Environment} &
\textbf{BC} &
\textbf{Offline } &
\textbf{Random} &
\textbf{Ours (GP)} &
\textbf{Reduction (\%)} \\
\hline
halfcheetah-random-v2& 2.3 & 13.5 & 38.2 & \textbf{39.1} & 41.3\\
hopper-random-v2 & 4.2 & 8.2 & 21.5 & \textbf{22.6} & 30.1\\
walker2d-random-v2 &  2.0&  7.9& 10.7 &  \textbf{10.8}& 34.4\\
halfcheetah-medium-v2 & 42.8 & 48.3 & 54.3 & \textbf{57.7}& 30.1\\
hopper-medium-v2 & 54.0 & 68.1 & 83.1 & \textbf{85.1} & 32.4\\
walker2d-medium-v2 & 73.1 & 83.6 & 84.9 & \textbf{85.2} & 42.2\\
\hline
maze2d-umaze-hard-v1 & -8.1 & -11.5 & 100.3 & \textbf{152.12} & 78.1 \\
maze2d-medium-easy-v1 & -4.5 & -4.2 & 59.6 &  \textbf{156.71}& 75.4 \\
maze2d-medium-hard-v1 & -4.4 & -3.7 & 41.8 &  \textbf{150.71}& 61.9 \\
maze2d-medium-dense-hard-v1 & -6.1 & -12.3 & 78.3 &  \textbf{119.2}& 89.2 \\
maze2d-large-easy-v1 & 1.7 & -3.6 & 15.6 & \textbf{120.5} & 54.0\\
maze2d-large-hard-v1 & -2.3 & -2.0 & 5.3 & \textbf{109.1} & 40.4\\
\hline
antmaze-umaze-v0 & 62.0 & 74.7 & 79.9 & \textbf{83.5} & 40.1\\
antmaze-umaze-diverse-v0 & 54.0 & 56.0 & 38.6 &  \textbf{80.6}& 51.6 \\
antmaze-medium-play-v0 & 0.0 & 58.1 & 51.3 & \textbf{71.8} & 40.3\\
antmaze-medium-diverse-v0 & 1.3 & 62.1 & 61.2 & \textbf{69.5} & 30.4\\

\hline
\end{tabular}
}
\end{table}

All the baselines use identical network architectures, training budgets, and dataset mixing procedures. The performance is evaluated using D4RL normalized scores and learning curves as a function of the active environment steps. For navigation tasks, we additionally visualize value and uncertainty heatmaps. The quantitative results are summarized in Table \ref{tab:combined_results}, with further analyses in Appendix \ref{appendix:exp}. 
\subsection{Implementation Summary} To ensure reproducibility, we highlight key hyperparameters here; full details are provided in Appendix~\ref{appendix:exp}.
\textbf{Architectures:} The control policies (Actor/Critic) are parameterized as 3-layer MLPs with 256 hidden units and ReLU activations. The SVGP inducing points are initialized via K-Means clustering on the offline dataset to ensure broad state-space coverage.
\textbf{Active Loop:} We employ an $\epsilon$-greedy exploration strategy with $\epsilon=0.2$. During uncertainty-guided steps, we sample $N=32$ candidate actions and select the one maximizing the predicted next-state variance.
To balance stability and speed, we perform lightweight policy updates every $10^3$ environment steps, while full GP retraining (via FVI) and policy consolidation occur every $5\times 10^3$ steps.
\vspace{-0.5em}

\begin{figure*}[t]
    \centering
    \includegraphics[
        width=0.7\linewidth
    ]{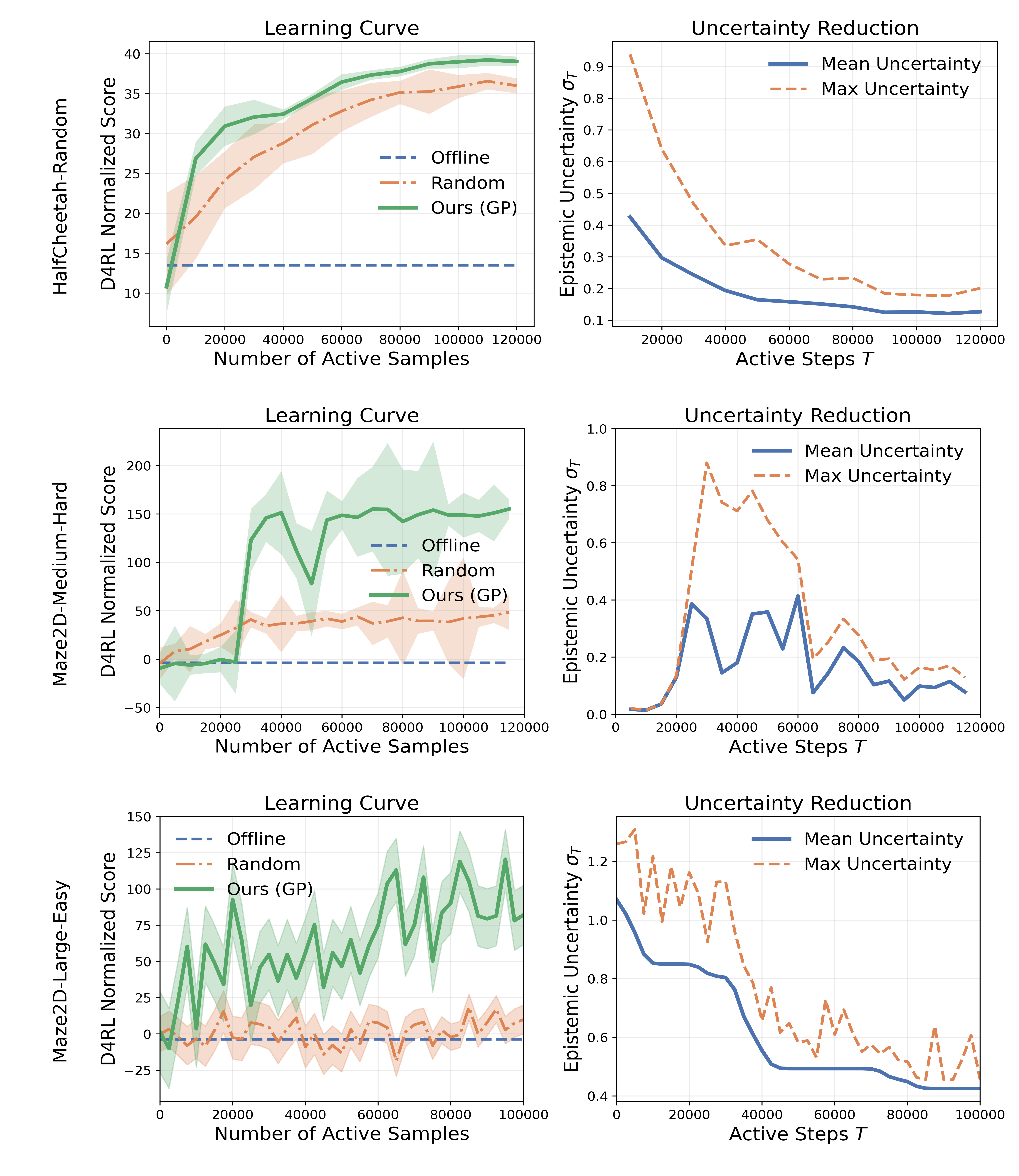}
    \caption{
    Learning performance and uncertainty reduction across different environments.
    \textbf{Top}: HalfCheetah-Random.
    \textbf{Middle}: Maze2D-Medium-Hard.
    \textbf{Bottom}: Maze2D-Large-Easy.
    }
    \label{fig:learning_uncertainty_combined}
\end{figure*}

%================================================================
\section{Results \& Discussion}
\begin{figure*}[t]
    \centering
    \includegraphics[
        width=0.7\linewidth
    ]{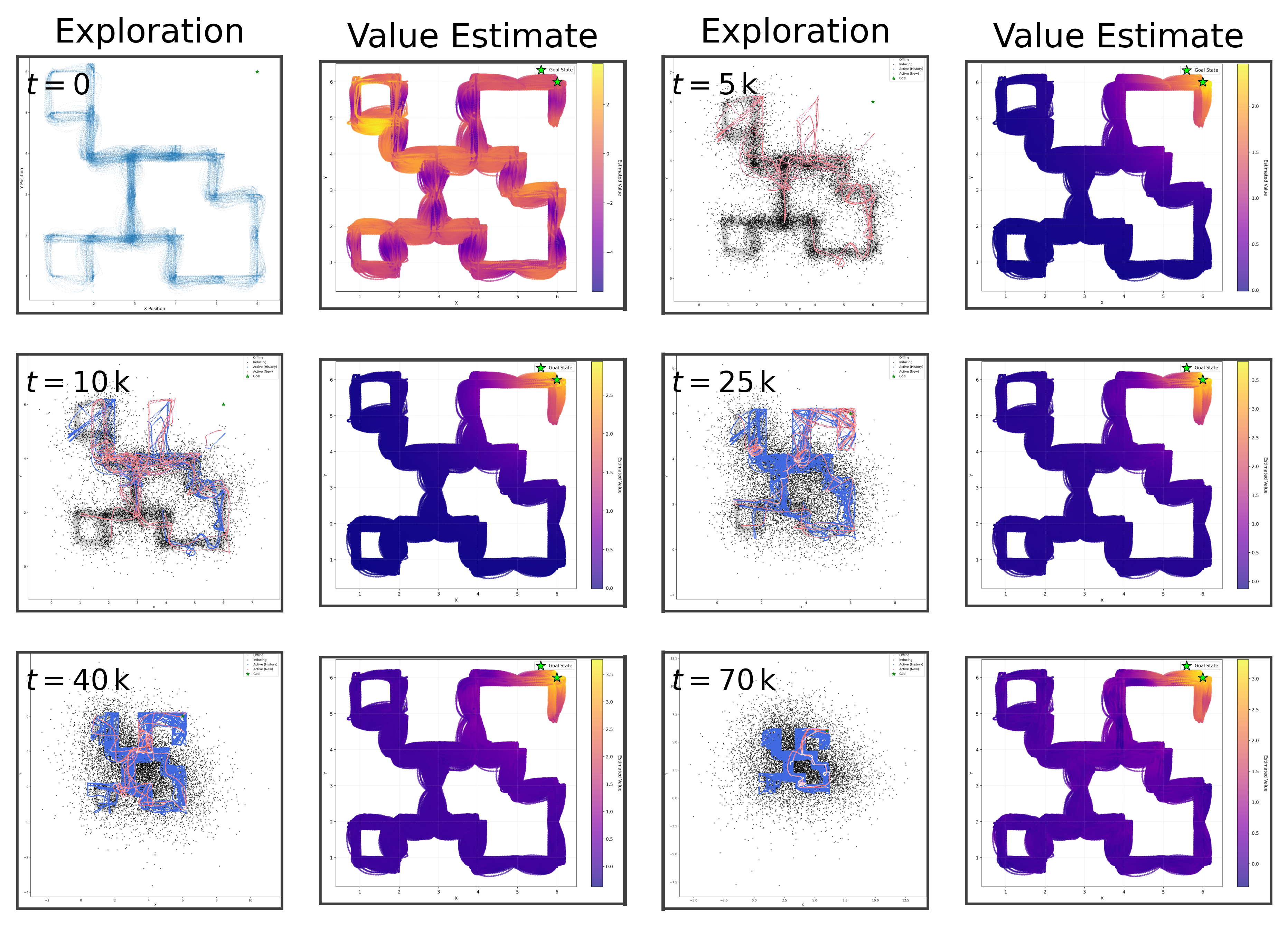}
    \caption{
    Uncertainty-driven exploration and value estimation in
    \texttt{maze2d-medium-hard}.
    For each timestep $t$, the exploration map (left) and the
    corresponding to the value function estimate (right). Value Estimates are calculated using the full offline dataset.
    }
    \label{fig:maze_medium_exploration_value_plot}
\end{figure*}

\begin{figure}[t]
    \centering
    \includegraphics[
        width=0.7\columnwidth
    ]{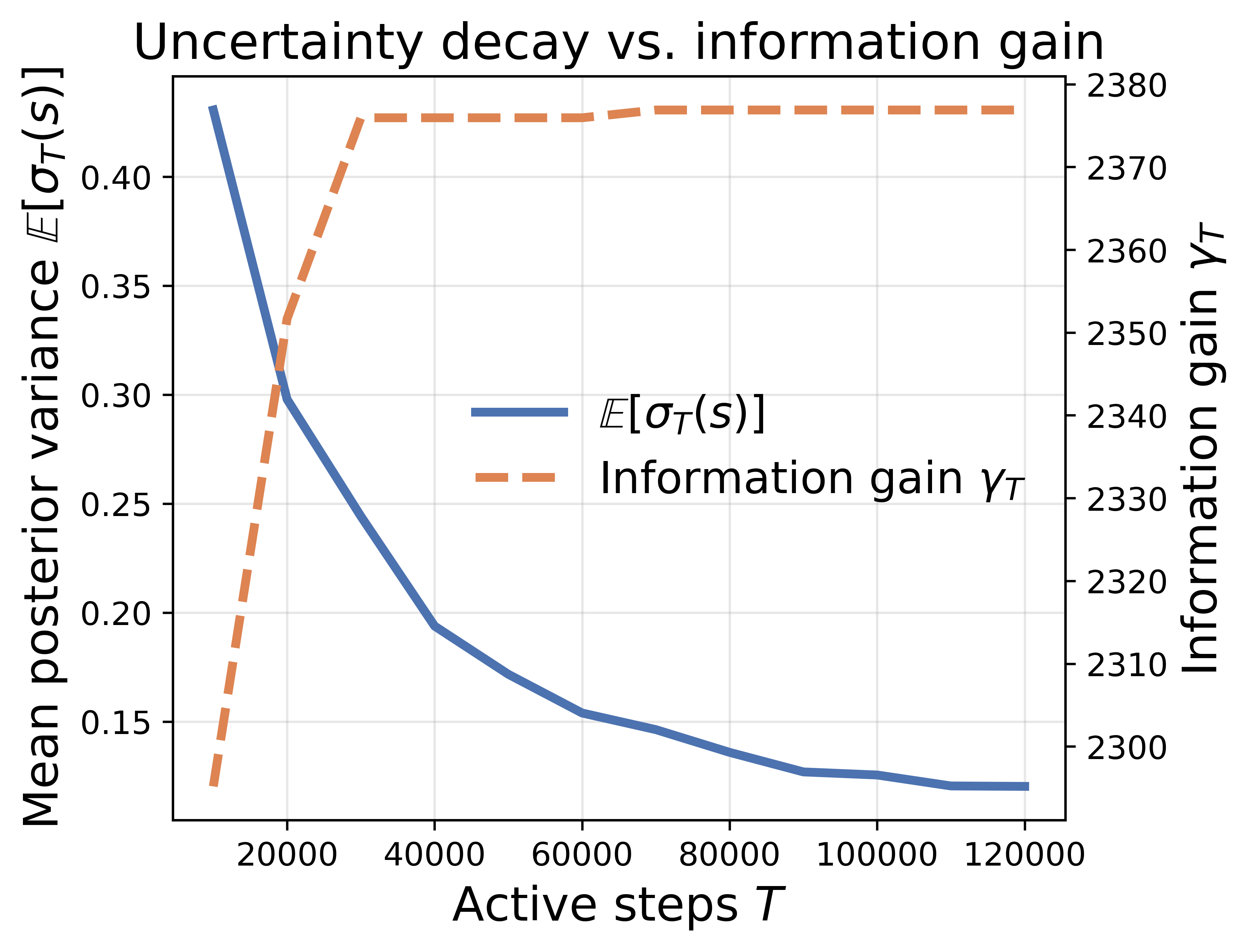}
    \caption{
\textbf{Uncertainty decay vs.\ information gain (halfcheetah-random-v2 ).}
The mean posterior variance decreases, while the cumulative information gain increases with active steps $T$, consistent with the GP concentration bounds $\sigma_T = O(\gamma_T / T)$.
}

    \label{fig:uncert_ig_plot}
\end{figure}

The results in Table~\ref{tab:combined_results} summarize the performance across all benchmarks.
Across both dense-reward locomotion and sparse-reward navigation tasks, uncertainty-aware active data collection consistently improves performance
over behavior cloning, purely offline RL, and random exploration baselines while requiring substantially fewer environment interactions.

\subsection{Performance and Sample Efficiency}
Across locomotion environments, our method matches or exceeds the strongest baseline while requiring substantially fewer samples. On \texttt{halfcheetah-medium-v2} and \texttt{hopper-medium-v2}, we obtained normalized scores of 57.7 and 85.1, respectively, with approximately 30--32\% fewer samples. Similar behavior is observed on \texttt{walker2d-medium-v2}, where comparable performance is achieved with over 40\% of the sample reduction. These results indicate that uncertainty-guided exploration remains effective
even in dense reward settings.

The gains are significantly larger in sparse-reward navigation tasks.
On \texttt{Maze2D}, our approach consistently outperforms all baselines, often by wide margins.
For example, in \texttt{maze2d-umaze-hard-v1} and
\texttt{maze2d-medium-easy-v1}, we observe improvements exceeding 50 normalized points over random exploration while reducing the total
number of samples by more than 75\%.
This highlights the failure of unguided exploration in long-horizon goal-conditioned tasks and the importance of targeting epistemically uncertain regions.

\subsection{Uncertainty-Driven Exploration and Value Propagation}
Figure~\ref{fig:maze_medium_exploration_value_plot} visualizes how uncertainty guides exploration in \texttt{maze2d-medium-hard}.
Early in training, high posterior variance concentrates near unexplored bottlenecks and dead-end areas.
As active data are collected, the uncertainty decreases locally, and the value function estimate propagates coherently through the maze.
This behavior aligns with the Bellman contraction argument: once uncertainty is reduced along critical transitions, value information propagates rapidly across long horizons.

\subsection{Uncertainty Decay and Learning Dynamics}
Figure~\ref{fig:learning_uncertainty_combined} compares learning curves and uncertainty reduction in a dense-reward task (\texttt{HalfCheetah-Random}) and a sparse-reward task (\texttt{Maze2D-Medium-Hard}).
In both cases, the performance improvements coincided with monotonic decreases in the GP posterior variance.
Notably, uncertainty decays more rapidly in dense-reward environments, where frequent reward feedback accelerates value estimation, while sparse-reward tasks exhibit slower but more structured uncertainty reduction concentrated around task-relevant regions.

\subsection{Information Gain and Theoretical Verification}
To directly validate the information-theoretic argument underlying
Theorem~\ref{thm:gp_sample_complexity}, we track both the mean posterior
variance $\mathbb{E}[\sigma_T(s)]$ and the cumulative information gain
$\gamma_T$ during active exploration.
As predicted, uncertainty decreases as a function of active steps,
while the information gain grows sublinearly with $T$.
This empirical relationship confirms that uncertainty-maximizing queries
effectively reduce posterior variance at a rate consistent with
$\mathcal{O}(\sqrt{\gamma_T/T})$, validating the role of information gain
in driving sample efficiency improvements  Figure 
\ref{fig:uncert_ig_plot},\ref{fig:uncert1_ig_plot}.

\subsection{Scalability and Practical Considerations}
Gaussian process models are not inherently scalable, and even sparse variational approximations introduce nontrivial computational overheads. In our experiments, we use between 5{,}000 and 25{,}000 inducing points for the tasks. While this setting is sufficient to produce reliable uncertainty estimates
and stable value learning, training costs remain higher than those of purely neural approaches.
Therefore, our goal is not to claim scalability parity with deep methods but to empirically validate the uncertainty and sample complexity guarantees derived in our analysis.

\begin{figure*}[t]
    \centering
    \includegraphics[
        width=0.99\linewidth
    ]{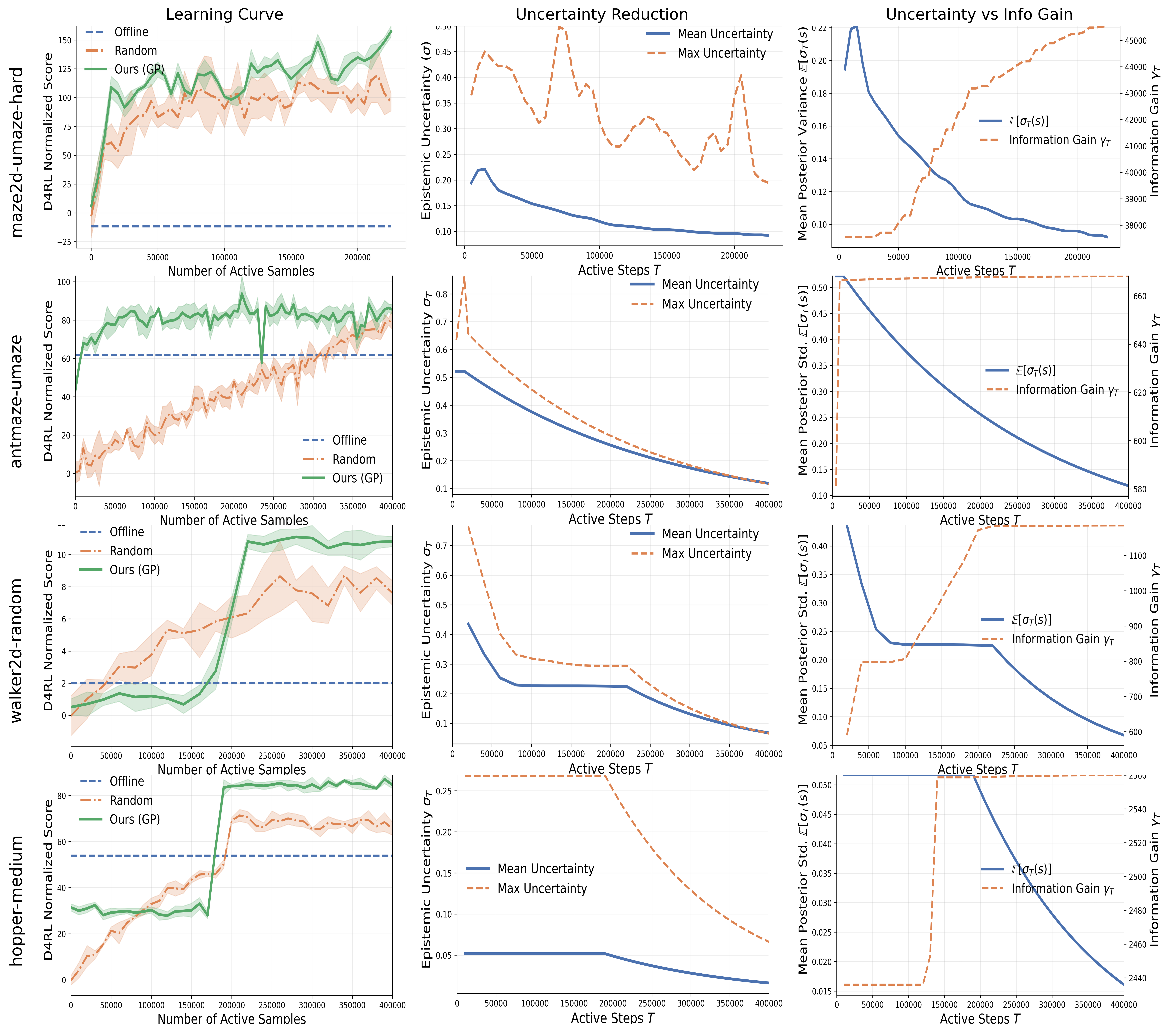}
    \caption{
        \textbf{Learning dynamics in different environments.}
        Each row corresponds to a different environment, and the columns show
        (\textbf{a}) the learning curve under active data collection,
        (\textbf{b}) the reduction of GP epistemic uncertainty over active steps,
        and (\textbf{c}) the relationship between posterior uncertainty and cumulative information gain.
    }
    \label{fig:uncert1_ig_plot}
\end{figure*}

%===============================
\section{Concluding Remarks}
The GP-based active offline RL algorithm should be viewed as a continuous, 
kernelized limit of the deep ActiveRL method. 
It idealizes the uncertainty estimation process by assuming exact Bayesian posteriors, 
This enables a rigorous analysis of uncertainty reduction and value-function convergence. 
In practice, ensemble-based uncertainty estimates approximate the same principle. 
regions of high epistemic uncertainty correspond to high GP variance. 
Hence, our theoretical results provide a formal justification for the empirical success 
of uncertainty-guided exploration in Active Offline RL.

% Acknowledgements should only appear in the accepted version.
%\section*{Acknowledgements}

%\section*{Impact Statement}
%This paper presents theoretical work whose primary goal is to advance the understanding of reinforcement learning algorithms in the context of offline-to-online reinforcement learning.
%While our results may inform the design of more sample-efficient learning methods. We do not believe that our work will have any significant societal consequences.

% In the unusual situation where you want a paper to appear in the
% references without citing it in the main text, use \nocite
\nocite{langley00}

\bibliography{references}
\bibliographystyle{icml2026}

%%%%%%%%%%%%%%%%%%%%%%%%%%%%%%%%%%%%%%%%%%%%%%%%%%%%%%%%%%%%%%%%%%%%%%%%%%%%%%%
%%%%%%%%%%%%%%%%%%%%%%%%%%%%%%%%%%%%%%%%%%%%%%%%%%%%%%%%%%%%%%%%%%%%%%%%%%%%%%%
% APPENDIX
%%%%%%%%%%%%%%%%%%%%%%%%%%%%%%%%%%%%%%%%%%%%%%%%%%%%%%%%%%%%%%%%%%%%%%%%%%%%%%%
%%%%%%%%%%%%%%%%%%%%%%%%%%%%%%%%%%%%%%%%%%%%%%%%%%%%%%%%%%%%%%%%%%%%%%%%%%%%%%%
\newpage
\appendix
\onecolumn

\begin{center}
    \text{\Large \textbf{Sample Efficient Active Algorithms for Offline Reinforcement Learning}} \\
    \text{\Large \textbf{Appendix}}
\end{center}

\section{Notation}
\label{sec:notation_quantities}

We summarize here the key quantities and constants appearing in 
\Cref{thm:gp_sample_complexity} and throughout the analysis.

\paragraph{Active Offline RL notation:}
\begin{itemize}[leftmargin=*, topsep=1pt]
    \item $J(\pi)$: the expected discounted return of a policy $\pi$.
    \item $\gamma \in (0,1)$: the discount factor of the MDP,
    This also determines the contraction rate of the Bellman operator.     
    \item $T$: number of active sampling rounds (each round adds one transition).
    \item $C$: a universal numerical constant that collects normalization and inequality constants.
    \item $\delta$: the confidence level for high-probability guarantees ($1-\delta$).
\end{itemize}

\paragraph{Lipschitz and smoothness constants:}
\begin{itemize}[leftmargin=*, topsep=1pt]
    \item $L_p$: the Lipschitz constant of the transition kernel $P(\cdot|s,a)$, defined in \Cref{ass:lipschitz_transitions}:
    \[
    \|P(\cdot|s,a) - P(\cdot|s',a')\|_1
    \le L_p(\|s-s'\| + \|a-a'\|).
    \]
    It measures the sensitivity of the next-state distributions to small changes in $(s,a)$.

    \item $L$: the induced Lipschitz constant of the Bellman operator $\mathcal{T}^\pi$, given by
    \[
    L = \gamma(1 + L_p),
    \]
    which quantifies how estimation errors in the value function 
    propagates through the Bellman update.
\end{itemize}

\paragraph{GP notations:}
\begin{itemize}[leftmargin=*, topsep=1pt]
    \item $\boldsymbol{B}$: the RKHS norm bound on the true value function, $\|V^*\|_k \le B$, from \Cref{ass:smoothness}.  
    It encodes the smoothness of $V^*$ relative to the Gaussian process (GP) kernel.

    \item $\sigma$: the standard deviation of the Gaussian observation noise $\eta_t \sim \mathcal{N}(0,\sigma^2)$,
    appeared in the GP regression model.

    \item $\beta_T$: the GP confidence width parameter,
    \[
    \beta_T = B + \sigma \sqrt{2(\gamma_T + \log(1/\delta))},
    \]
    This determines how tightly the GP posterior concentrates around the true function.

    \item $\gamma_T$: the \emph{maximum information gain} after $T$ samples,
    \[
    \gamma_T = \max_{|A|=T} I(\mathbf{y}_A;V)
    = \frac{1}{2}\log\!\big|I + \sigma^{-2}\mathbf{K}_A\big|.
    \]
    It measures the extent to which posterior uncertainty can be reduced after $T$ observations.
    
\end{itemize}

\section{Detailed Assumptions}
\label{appendix:assumptions}

\label{sec:assumptions}
\begin{assumption}[Bounded Rewards]
\label{ass:bounded_rewards}
The reward function is uniformly bounded:
$r(s,a) \in [0, R_{\max}]$ for all $(s,a) \in \Sspace \times \Aspace$.
\end{assumption}

\begin{assumption}[RKHS Smoothness]
\label{ass:smoothness}
The optimal value function satisfies $V^* \in \mathcal{H}_k$ with
$\|V^*\|_k \le B$, where $\mathcal{H}_k$ is the RKHS induced by kernel $k$.
\end{assumption}

\begin{assumption}[Lipschitz Transitions]
\label{ass:lipschitz_transitions}
The transition kernel $P(\cdot \mid s,a)$ is $L_p$-Lipschitz in $(s,a)$:

$ \|P(\cdot \mid s,a) - P(\cdot \mid s',a')\|_1
\le L_p \bigl(\|s-s'\| + \|a-a'\|\bigr),
\quad \forall (s,a),(s',a')$.

\end{assumption}

\begin{assumption}[Bounded Observation Noise]
\label{ass:noise}
Observed Bellman targets satisfy
$y_t = r_t + \gamma \mu_{t-1}(s_t') + \eta_t$,
where $\eta_t \sim \mathcal{N}(0,\sigma^2)$ i.i.d.\ with known $\sigma>0$.
\end{assumption}

\begin{definition}[Offline Coverage Radius]
\label{def:offline_coverage_measure}
Let $\sigma_0(s)$ denote the GP posterior standard deviation after conditioning
on the offline dataset $D_{\mathrm{off}}$.
Define $\sigma_{\max} := \max_{s \in \Sspace} \sigma_0(s).$
\end{definition}

\begin{assumption}[Bounded Offline Coverage]
\label{ass:offline_coverage}
The offline coverage radius is finite: $\sigma_{\max} < \infty$.
\end{assumption}
For any fixed policy $\pi$, the Bellman operator $\mathcal{T}^\pi$ acting on bounded value functions is given by
\[
(\mathcal{T}^\pi V)(s)
= r(s,\pi(s)) + \gamma\,\mathbb{E}_{s'\sim P(\cdot \mid s,\pi(s))}[V(s')].
\]
Moreover, $\mathcal{T}^\pi$ is a $\gamma$-contraction under the sup-norm.

\begin{definition}[Information Gain]
\label{def:information_gain}
Let $f : \Sspace \to \mathbb{R}$ be a latent function and $\mathcal{D}_A = \{(s_i, y_i)\}_{i=1}^{|A|}$ 
be noisy observations at input locations $A \subseteq \Sspace$, 
where $y_i = f(s_i) + \eta_i$ with independent noise $\eta_i \sim \mathcal{N}(0,\sigma^2)$.
The \emph{information gain} from observing $\mathbf{y}_A$ about $f$ is defined as mutual information:
$I(\mathbf{y}_A; f) = H(\mathbf{y}_A) - H(\mathbf{y}_A \mid f)$.

where $H(\cdot)$ denotes the Shannon differential entropy.
The \emph{maximum information gain} after $T$ observations is
$\gamma_T = \max_{A \subseteq \Sspace,\, |A| = T} I(\mathbf{y}_A; f).$
This quantity measures how much observing $T$ noisy samples can reduce the uncertainty about the latent function $f$.
The rate of growth of $\gamma_T$ depends on the smoothness properties of the underlying function class or kernel.
\end{definition}

\begin{definition}[Information Gain in Gaussian Process Regression]
\label{def:gp_information_gain}
Let $V \sim \mathcal{GP}(0, k)$ be a Gaussian Process prior with kernel $k$, 
and let $\mathbf{y}_A = [y_1, \ldots, y_T]^\top$ denote noisy observations of $V$ at points $A = \{s_1, \ldots, s_T\}$, 
where $y_i = V(s_i) + \eta_i$, $\eta_i \sim \mathcal{N}(0, \sigma^2)$.
Then the mutual information between the function $V$ and the observations $\mathbf{y}_A$ has the closed-form expression
$I(\mathbf{y}_A; V)
= \frac{1}{2} \log \big| I + \sigma^{-2} \mathbf{K}_A \big|$,
where $\mathbf{K}_A$ is the kernel matrix with entries $[\mathbf{K}_A]_{ij} = k(s_i,s_j)$.
The \emph{maximum information gain} after $T$ samples is
$\gamma_T = \max_{A \subseteq \Sspace,\, |A| = T} \frac{1}{2} \log \big| I + \sigma^{-2} \mathbf{K}_A \big|.$ The term $\gamma_T$ quantifies the maximum reduction in GP posterior uncertainty achievable with $T$ observations.
It depends solely on the choice of kernel $k$ and the input dimension $d$,
and plays a central role in controlling the cumulative posterior variance and sample complexity.
\end{definition}

\begin{lemma}[Bounded Information Gain for Smooth Kernels]
\label{lem:bounded_info_gain}
For commonly used stationary kernels, the maximum information gain
$\gamma_T$ satisfies
$\gamma_T = \mathcal{O}\!\left(g_k(T,d)\right)$,
where $g_k(T,d)$ depends on the kernel smoothness and the input
dimension $d$. In particular, $\gamma_T = \mathcal{O}\left((\log T)^{d+1}\right)$, for RBF, $\gamma_T = \mathcal{O}\left(T^{\frac{d}{2\nu + d}} \log T\right)$ for Mat\'ern, 
$\gamma_T = \mathcal{O}\left(d \log T\right)$ for linear kernels. 
\end{lemma}

This result follows from the spectral decay of the kernel-integral operator. For the squared-exponential kernel, the eigenvalues decay super-polynomially,
yielding $\gamma_T = \mathcal{O}((\log T)^{d+1})$ and ensuring sublinear uncertainty reduction.

These assumptions collectively establish a smooth, well-behaved environment in which the GP posteriors are calibrated and the Bellman updates are stable.  
They are mild and standard in GP learning theory and kernelized RL \citep{chowdhury2017kernelized} 
yet sufficient to yield our main ${\mathcal{O}}(1/\epsilon^2)$ sample complexity result.

\section{Key Lemmas}
\label{sec:key_lemmas}

This section presents the key intermediate results used to prove 
\Cref{thm:gp_sample_complexity}.  
Each lemma isolates a distinct component of the analysis—concentration, variance control, 
Bellman error propagation and performance improvement.
The full proofs are deferred to Appendix~\ref{appendix:proofs}.

\begin{lemma}[GP Concentration Bound]
\label{lem:gp_concentration}
Under \Cref{ass:smoothness} - \Cref{ass:noise}, with probability at least $1-\delta$, 
for all $t \ge 1$ and $s \in \Sspace$,
\[
|V^*(s) - \mu_{t-1}(s)| 
\le 
\big(B + \sigma\sqrt{2(\gamma_{t-1} + \log(1/\delta))}\big)\,
\sigma_{t-1}(s).
\]
\textbf{Role:} This lemma bounds the estimation error of the GP posterior mean in terms of the predictive variance.  
It establishes a high-probability confidence envelope around $V^*$ that underpins the entire uncertainty-driven exploration argument.  
\textbf{Proof sketch:} Follows directly from the GP-UCB confidence bound 
\citep[Theorem~2]{srinivas2009gaussian} applied to the value function $V^*$.
A complete proof is provided in Appendix~\ref{appendix:proofs}.
\end{lemma}

\begin{lemma}[Variance–Information Gain Relationship]
\label{lem:var_sum}
For any sequence of queried states $\{s_t\}_{t=1}^T$, the cumulative predictive variance satisfies
\[
\sum_{t=1}^{T}\sigma_{t-1}^2(s_t) \le C_\sigma\,\gamma_T,
\]
where $\gamma_T$ is the maximum information gain defined in \Cref{ass:offline_coverage}, 
and $C_\sigma$ depend only on the kernel $k$ and noise variance $\sigma^2$.
\textbf{Role:} This lemma connects posterior uncertainty reduction to the information gain, 
This shows that active queries in high-variance regions achieve sublinear cumulative uncertainty.  
\textbf{Proof sketch:} Follows from the variance-sum inequality in 
\citet[Lemma~5]{srinivas2009gaussian}, restated for our RL setting.
\end{lemma}

\begin{lemma}[Lipschitz Propagation of Bellman Error]
\label{lem:bellman_lip}
Under \Cref{ass:lipschitz_transitions}, 
the Bellman operator $\mathcal{T}^\pi$ satisfies
\[
|\mathcal{T}^\pi V(s) - \mathcal{T}^\pi \tilde{V}(s)|
\le 
\gamma(1+L_p)\|V - \tilde{V}\|_\infty,
\quad \forall V,\tilde{V}:\Sspace\to\mathbb{R}.
\]
\textbf{Role:} This lemma bounds how estimation errors in the value function propagate through Bellman updates, 
This provides geometric control over the error amplification across iterations.  
\textbf{Proof sketch:} Derived using the contraction property of $\mathcal{T}^\pi$ 
and the $L_p$–Lipschitz continuity of $P(\cdot|s,a)$; full proof in Appendix~\ref{appendix:proofs}.
\end{lemma}

\begin{lemma}[Expected Performance Difference]
\label{lem:performance_diff}
Let $\pi_T$ denote the policy after $T$ active rounds and $\mu_T$ be the corresponding GP mean estimate.  
Then, under \Cref{ass:bounded_rewards},
\[
J(\pi^*) - J(\pi_T)
\le
\left(\beta_T + \frac{L}{1-\gamma}\right)
\mathbb{E}_{s\sim d_\rho^{\pi_T}}[\sigma_T(s)],
\]
where $d_\rho^{\pi_T}$ denotes the discounted state distribution under $\pi_T$.
\textbf{Role:} Links GP uncertainty $\sigma_T$ directly to policy suboptimality, 
bridging the statistical and control aspects of analysis.  
\textbf{Proof sketch:} Combines the GP concentration bound (\Cref{lem:gp_concentration}) 
with the Bellman Lipschitz property (\Cref{lem:bellman_lip}); 
see Appendix~\ref{appendix:proofs}.
\end{lemma}

\Cref{lem:gp_concentration} provides the high-probability confidence region for the GP posterior;
\Cref{lem:var_sum} ensures sublinear uncertainty reduction via information gain;
\Cref{lem:bellman_lip} guarantees controlled error propagation through the Bellman operator; 
and \Cref{lem:performance_diff} translates these into a bound on the policy performance.  
Together, these yield the main sample-complexity bound in 
\Cref{thm:gp_sample_complexity}.

\section{Proof of Main Theorem}
\label{sec:proofs}

This section presents detailed proofs of our main theoretical result 
(\Cref{thm:gp_sample_complexity}).  
Each proof follows a modular structure, invoking the lemmas established in 
\Cref{sec:key_lemmas} and the assumptions stated in 
\Cref{sec:assumptions}.  
For clarity, we first provide a brief roadmap outlining the logic of the argument.

The proof proceeds in four steps.
\begin{enumerate}[leftmargin=*]
    \item \textbf{Error decomposition:} Express the suboptimality gap as the sum of estimation and propagation errors.
    \item \textbf{Lemma application:} Use GP concentration (\Cref{lem:gp_concentration}) and Bellman Lipschitz propagation (\Cref{lem:bellman_lip}) to bound these errors.
    \item \textbf{Variance control:} Apply the variance–information gain relationship (\Cref{lem:var_sum}) to express expected uncertainty reduction as a function of the number of active queries $T$.
    \item \textbf{Combining results:} Put all bounds to obtain the final rate, and concluding the proof of \Cref{thm:gp_sample_complexity}.
\end{enumerate}

\begin{proof}[Proof of \Cref{thm:gp_sample_complexity}]
\textbf{Step 1: Decomposition of Suboptimality.}
From the definition of expected return (\Cref{eq:objective}), we express the suboptimality gap as
\[
J(\pi^*) - J(\pi_T) 
= 
\E_{s\sim\rho}[V^*(s) - V^{\pi_T}(s)].
\]
Following the standard approach, we decompose the gap into two terms:
\[
V^*(s) - V^{\pi_T}(s)
=
[V^*(s) - \mu_T(s)]
+
[\mu_T(s) - V^{\pi_T}(s)].
\]
The first term corresponds to the \emph{estimation error} between the true and GP-estimated value, 
The second represents the \emph{propagation error} due to imperfect Bellman updates under $\pi_T$.

\textbf{Step 2: Bounding the estimation and propagation terms.}
By the GP concentration inequality (\Cref{lem:gp_concentration}) and by \Cref{ass:smoothness}, 
we have, with probability at least $1-\delta$,
\[
|V^*(s) - \mu_T(s)| \le \beta_T \sigma_T(s),
\quad 
\beta_T = B + \sigma \sqrt{2(\gamma_T + \log(1/\delta))}.
\]
Next, invoking  lemma \ref{lem:bellman_lip} 
,
we obtain
\[
|\mu_T(s) - V^{\pi_T}(s)|
\le 
\frac{L}{1-\gamma}\,\sigma_T(s),
\]
where $L = \gamma(1+L_p)$ by \Cref{ass:lipschitz_transitions}.
Taking expectations over $s \sim d_\rho^{\pi_T}$ yields
\begin{equation}
\label{eq:proof_decomp}
J(\pi^*) - J(\pi_T)
\le
\Big(\beta_T + \frac{L}{1-\gamma}\Big)
\E_{s\sim d_\rho^{\pi_T}}[\sigma_T(s)].
\end{equation}

\textbf{Step 3: Variance–Information Gain Relationship.}
By the variance–information gain lemma (\Cref{lem:var_sum}) 
and bounded information gain assumption (\Cref{ass:offline_coverage}),
the cumulative posterior variance satisfies
\[
\sum_{t=1}^{T} \sigma_{t-1}^2(s_t) 
\le 
C_\sigma\,\gamma_T.
\]
By Jensen’s inequality, we have
\[
\E[\sigma_T(s)] 
\le 
\sqrt{\E[\sigma_T^2(s)]}
\le 
\sqrt{\frac{C_\sigma \gamma_T}{T}}.
\]
Substituting this into \Cref{eq:proof_decomp} gives
\begin{equation}
\label{eq:proof_bound}
J(\pi^*) - J(\pi_T)
\le
C'
\left(\beta_T + \frac{L}{1-\gamma}\right)
\sqrt{\frac{\gamma_T}{T}},
\end{equation}
where $C'$ absorbs constants from $C_\sigma$ and the expectations over the data distribution.

\textbf{Step 4: Concluding the Bound.}
Rearranging \Cref{eq:proof_bound} to ensure $\epsilon$-optimality 
($J(\pi^*) - J(\pi_T) \le \epsilon$) 
yields the required number of active rounds:
\[
T
\ge
\tilde{\mathcal{O}}\!\left(
\frac{
(\beta_T + \frac{L}{1-\gamma})^2\,\gamma_T
}{\epsilon^2}
\right).
\]
Since each active round corresponds to at most $\bar{H} \le (1-\gamma)^{-1}$ transitions,
the total number of active transitions required is
\[
M = T \cdot \bar{H}
= 
\tilde{\mathcal{O}}\!\left(
\frac{1}{\epsilon^2 (1-\gamma)^2}
\right),
\]
establishing the claim in \Cref{thm:gp_sample_complexity}.
\end{proof}

The proof demonstrates how GP uncertainty contracts with active querying and how this contraction translates into an improved policy value.  
The key novelty lies in connecting the information-gain-based posterior variance decay (\Cref{lem:var_sum})
to Bellman stability via Lipschitz propagation (see \Cref{lem:bellman_lip}).
This coupling enables us to achieve a $\tilde{\mathcal{O}}(1/\epsilon^2)$ rate 
while improving horizon dependence from $(1-\gamma)^{-4}$ (offline RL) 
to $(1-\gamma)^{-2}$ (active offline RL).
This argument unifies the GP learning theory with the RL sample complexity analysis in a nonparametric setting.

\section{Full Proofs}
\label{appendix:proofs}

This appendix provides detailed proofs for the lemmas and main theorem stated in 
\Cref{sec:key_lemmas,sec:proofs}.  
All notations follow \Cref{sec:notation_quantities}, and the assumptions are in \Cref{sec:assumptions}.

\paragraph{Performance Difference Lemma.}
For any two stationary policies $\pi$ and $\pi'$, the difference in their expected discounted returns satisfies
\begin{equation}
\label{eq:perf_diff_lemma}
J(\pi') - J(\pi)
= \frac{1}{1 - \gamma}\,
\mathbb{E}_{s \sim d_\rho^{\pi'}} 
\!\left[
(\mathcal{T}^{\pi'} V^\pi)(s) - V^\pi(s)
\right],
\end{equation}
where $d_\rho^{\pi'}(s) = (1 - \gamma) \sum_{t=0}^{\infty} \gamma^t 
P(s_t = s \mid s_0 \sim \rho, \pi')$ 
denotes the discounted state visitation distribution of $\pi'$,
and $\mathcal{T}^{\pi'}$ is the Bellman operator associated with $\pi'$.

\noindent
\textbf{Intuition:}
This lemma expresses the expected return difference between two policies in terms of the Bellman residual of $V^\pi$ under new policy $\pi'$.
This is a standard result in policy evaluation and serves as the foundation for many regret and sample complexity analyses in reinforcement learning.

\subsection{Proof of Lemma~\ref{lem:gp_concentration} (GP Concentration Bound)}
We restate this lemma for clarity.

\begin{lemma}[GP Concentration Bound]
Under \Cref{ass:smoothness} - \Cref{ass:noise}, with probability at least $1-\delta$, 
for all $t \ge 1$ and $s \in \Sspace$,
\[
|V^*(s) - \mu_{t-1}(s)| 
\le 
\big(B + \sigma\sqrt{2(\gamma_{t-1} + \log(1/\delta))}\big)\,
\sigma_{t-1}(s).
\]
\end{lemma}

\begin{proof}
The proof follows directly from the GP-UCB concentration inequality 
\citep{srinivas2009gaussian}.  
Let $V^* \in \mathcal{H}_k$ denote the true value function with $\|V^*\|_k \le B$ by \Cref{ass:smoothness}.
For observations $\{(s_i, y_i)\}_{i=1}^{t-1}$ satisfying the Gaussian noise model (\Cref{ass:noise}),  
the posterior mean and variance of the GP satisfy:
\[
\mu_{t-1}(s) = \mathbf{k}_{t-1}(s)^\top(\mathbf{K}_{t-1} + \sigma^2 I)^{-1}\mathbf{y}_{t-1}, \quad
\sigma_{t-1}^2(s) = k(s,s) - \mathbf{k}_{t-1}(s)^\top(\mathbf{K}_{t-1} + \sigma^2 I)^{-1}\mathbf{k}_{t-1}(s).
\]
By standard concentration for GPs, with probability at least $1-\delta$, for all $s$,
\[
|V^*(s) - \mu_{t-1}(s)| 
\le
\beta_t \sigma_{t-1}(s),
\quad
\beta_t = B + \sigma\sqrt{2(\gamma_{t-1} + \log(1/\delta))}.
\]
This result follows from bounding the deviation of the posterior mean from the true function in the RKHS norm.  
The argument uses the reproducing property $\langle V^*, k(s,\cdot)\rangle_k = V^*(s)$ and 
the mutual information $\gamma_{t-1}$ as a measure of the model complexity.
\end{proof}

\subsection{Proof of Lemma~\ref{lem:var_sum} (Variance–Information Gain Relationship)}
\begin{lemma}[Variance–Information Gain Relationship]
For any sequence of queried states $\{s_t\}_{t=1}^T$, 
\[
\sum_{t=1}^{T}\sigma_{t-1}^2(s_t) \le C_\sigma\,\gamma_T.
\]
\end{lemma}

\begin{proof}
This follows from the information-gain identity for Gaussian Processes 
\citep{srinivas2009gaussian}.  
Define $\gamma_T = \max_{A \subset \Sspace, |A|=T} I(\mathbf{y}_A; V)$, 
where $I(\mathbf{y}_A; V)$ is the mutual information between the latent function $V$ and the noisy observations $\mathbf{y}_A$.  
For the GP model with observation noise $\sigma^2$, we have
\[
I(\mathbf{y}_A; V)
= \frac{1}{2} \log\!\left| I + \sigma^{-2}\mathbf{K}_A \right|.
\]
By the properties of mutual information, 
each new observation at $s_t$ contributes at most $\tfrac{1}{2}\log(1+\sigma_{t-1}^2(s_t)/\sigma^2)$ bits of information.  
Summing over $T$ rounds and using the inequality $\log(1+x) \ge x/(1+x)$ gives
\[
\sum_{t=1}^{T}\min\{1, \sigma_{t-1}^2(s_t)\} \le 2\log\!\left|I + \sigma^{-2}\mathbf{K}_T\right| = 2\gamma_T.
\]
Setting $C_\sigma=2$ and dropping the truncation by the boundedness of $\sigma_t(s)\le 1$ completes the proof.
\end{proof}

\subsection{Proof of Lemma~\ref{lem:bellman_lip} (Lipschitz Propagation of Bellman Error)}
\begin{lemma}[Lipschitz Propagation of Bellman Error]
Under  \Cref{ass:lipschitz_transitions}, 
\[
|\mathcal{T}^\pi V(s) - \mathcal{T}^\pi \tilde{V}(s)|
\le 
\gamma(1+L_p)\|V - \tilde{V}\|_\infty.
\]
\end{lemma}

\begin{proof}
Recall that $\mathcal{T}^\pi V(s) = \E_{a\sim\pi(s)}[r(s,a) + \gamma \E_{s'\sim P(\cdot|s,a)}[V(s')]]$.  
By the triangle inequality and \Cref{ass:lipschitz_transitions},
\begin{align*}
|\mathcal{T}^\pi V(s) - \mathcal{T}^\pi \tilde{V}(s)|
&= \gamma\Big|
\E_{a\sim\pi(s)} \E_{s'\sim P(\cdot|s,a)}[V(s') - \tilde{V}(s')]
\Big| \\
&\le \gamma\E_{a\sim\pi(s)} \E_{s'\sim P(\cdot|s,a)}|V(s') - \tilde{V}(s')| \\
&\le \gamma(1+L_p)\|V - \tilde{V}\|_\infty,
\end{align*}
where $(1+L_p)$ accounts for perturbations in $P(\cdot|s,a)$ under nearby states and actions.
This establishes the desired Lipschitz constant for $\mathcal{T}^\pi$.
\end{proof}

\subsection{Proof of Lemma~\ref{lem:performance_diff} (Expected Performance Difference)}
\begin{lemma}[Expected Performance Difference]
Let $\pi_T$ denote the policy after $T$ active rounds and $\mu_T$ the corresponding GP mean estimate.  
Then,
\[
J(\pi^*) - J(\pi_T)
\le
\left(\beta_T + \frac{L}{1-\gamma}\right)
\mathbb{E}_{s\sim d_\rho^{\pi_T}}[\sigma_T(s)].
\]
\end{lemma}

\begin{proof}
Using the performance difference lemma (\Cref{eq:perf_diff_lemma}),
\[
J(\pi^*) - J(\pi_T)
= \frac{1}{1-\gamma}\E_{s\sim d_\rho^{\pi_T}}[(\mathcal{T}^{\pi^*}V^{\pi_T})(s) - V^{\pi_T}(s)].
\]
We decompose the integrand as
\[
(\mathcal{T}^{\pi^*}V^{\pi_T})(s) - V^{\pi_T}(s)
= (\mathcal{T}^{\pi^*}V^*)(s) - \mathcal{T}^{\pi^*}\mu_T(s)
+ \mathcal{T}^{\pi^*}\mu_T(s) - \mathcal{T}^{\pi_T}\mu_T(s)
+ \mathcal{T}^{\pi_T}\mu_T(s) - V^{\pi_T}(s).
\]
The first difference is bounded by the GP concentration bound (\Cref{lem:gp_concentration}) as $\beta_T\sigma_T(s)$.  
The remaining two terms are controlled by the Lipschitz continuity of $\mathcal{T}^\pi$ (\Cref{lem:bellman_lip}),  
giving an additional $\tfrac{L}{1-\gamma}\sigma_T(s)$.  
Considering expectations completes the result.
\end{proof}

\subsection{Proof of Theorem~\ref{thm:gp_sample_complexity}}
We restate the main theorem for convenience as follows:

\begin{theorem}\textbf{Sample Complexity of \Cref{alg:gp_active_rl}}
Under \Cref{ass:bounded_rewards}--\Cref{ass:offline_coverage}, 
the policy $\pi_T$ obtained after $T$ active rounds satisfies, with probability at least $1-\delta$,
\[
J(\pi^*) - J(\pi_T)
\le
C\left(
\beta_T + \frac{L}{1-\gamma}
\right)
\sqrt{\frac{\gamma_T}{T}}.
\]
\end{theorem}

\begin{proof}
Combining \Cref{lem:performance_diff} with the variance bound in \Cref{lem:var_sum}, we obtain
\[
J(\pi^*) - J(\pi_T)
\le
\left(\beta_T + \frac{L}{1-\gamma}\right)
\E[\sigma_T(s)]
\le
\left(\beta_T + \frac{L}{1-\gamma}\right)
\sqrt{\frac{C_\sigma \gamma_T}{T}}.
\]
To ensure $\epsilon$-optimality ($J(\pi^*) - J(\pi_T) \le \epsilon$), it suffices that
\[
T = \tilde{\mathcal{O}}\!\left(
\frac{(\beta_T + \tfrac{L}{1-\gamma})^2 \gamma_T}{\epsilon^2}
\right).
\]
Substituting $\gamma_T = \mathcal{O}((\log T)^{d+1})$ for the squared-exponential kernel yields
$M = T \cdot \bar{H} = \tilde{\mathcal{O}}(1/[\epsilon^2 (1-\gamma)^2])$, 
as stated in \Cref{thm:gp_sample_complexity}.
\end{proof}

%\paragraph{Conclusion.}
%The proofs collectively establish that uncertainty-guided active exploration under a GP prior 
%achieves near-optimal sample efficiency.  
%By leveraging smoothness (\Cref{ass:smoothness}), bounded information gain (\Cref{ass:bounded_info_gain}), 
%and Lipschitz stability (\Cref{ass:lipschitz_transitions}), 
%the analysis yields a unified information-theoretic characterization of Active Offline RL.

\section{Experiments}
\label{appendix:exp}
\subsection{Offline Dataset Pruning }
To systematically study the performance under varying degrees of offline data coverage, we constructed pruned versions of each environment with controlled difficulty levels.

\textbf{Navigation Tasks (Maze2D, AntMaze)}\\
For navigation environments, we created Easy, Medium, and Hard variants by applying progressively stronger region-based pruning. Axis-aligned rectangular regions are defined in the 2D state space, and trajectories entering these regions are either removed or truncated at the first entry. Truncated trajectories are labeled as timeouts to preserve the validity of the Bellman targets.\\
After region pruning, we applied episode-level random pruning to further control the dataset size while maintaining temporal coherence. Across all navigation tasks, the resulting offline datasets retained approximately 7–10\% of the original data, with lower fractions corresponding to more difficult variants.\\
\textbf{Locomotion Tasks (HalfCheetah, Hopper, Walker2D)}\\
For locomotion environments, we applied uniform random episode pruning because the spatial structure was not meaningful. We retained 20\% of the original dataset for all locomotion tasks, preserving trajectory consistency and behavioral diversity.\\
No reward relabeling, trajectory stitching, or synthetic data were introduced in any of the settings. All active data were collected online from the environment.\\

\begin{figure}[H]
    \centering
    \includegraphics[
        width=0.95\textwidth
    ]{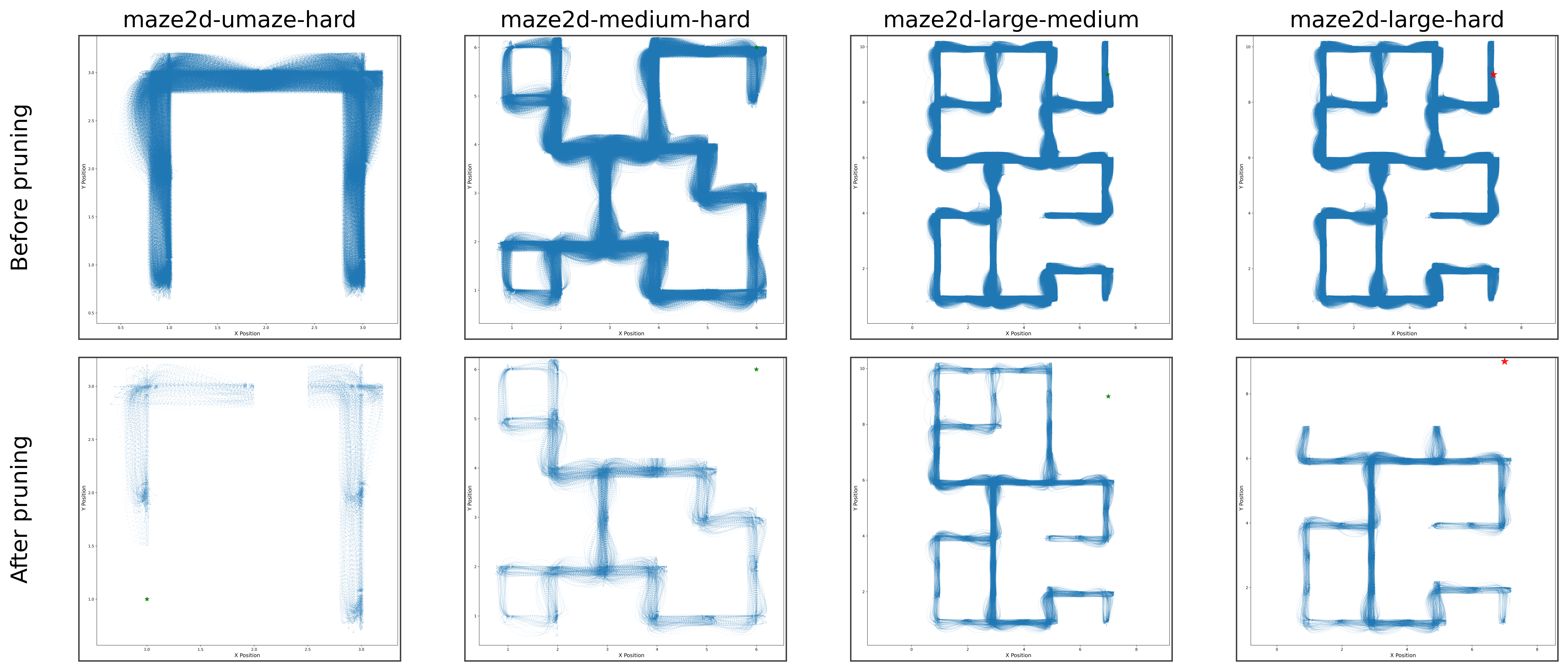}
    \caption{
    State visitation distributions before (top) and after (bottom) dataset pruning
    for four Maze2D environments:
    \texttt{maze2d-umaze-hard},
    \texttt{maze2d-medium-hard},
    \texttt{maze2d-large-medium},
    and \texttt{maze2d-large-hard}.
    Although pruning retains only a small fraction of trajectories,
    it preserves coverage of task-relevant regions while removing redundant
    and low-utility states.
    }
    \label{fig:state_distribution}
\end{figure}

\subsection{Gaussian Process Architecture and Training}
\subsubsection{Model Architecture}
We model the value function $V(s)$ using a sparse variational Gaussian process with a zero mean prior. The kernel is chosen from either the squared exponential (RBF) family or the Matérn family with a smoothness parameter $\nu = 2.5$. The Matérn kernel is used by default because it provides a better inductive bias for non-smooth value landscapes.
Let $\mathcal{Z} = {z_m}_{m=1}^M$ denote the inducing  The GP posterior is approximated using a variational distribution over inducing variables,
\[
q(U) = N(m,S)
\]
where $\mathbf{u} = V(\mathcal{Z})$. All inducing locations are learned jointly with the kernel hyperparameters.
\subsubsection{Inducing Point Initialization}
The inducing points are initialized using MiniBatch K-means clustering over normalized state observations from the offline dataset. This initialization accelerates early stage convergence compared to that of random initialization.

The number of inducing points is environment-dependent and is chosen to balance the approximation quality and computational cost. Values range from several thousand for navigation tasks to up to ten thousand for locomotion tasks.
\subsubsection{Variational Inference and Optimization}
Posterior inference is performed by maximizing the variational evidence lower bound (ELBO),
\[
\mathcal{L}_{\mathrm{ELBO}}
= \mathbb{E}_{q(f)} \big[ \log p(y \mid f) \big]
- \mathrm{KL}\!\left( q(u) \,\|\, p(u) \right).
\]
We optimize the variational parameters using natural-gradient descent (NGD), which provides stable and efficient updates in the space of Gaussian distributions. The kernel hyperparameters and observation noise were optimized using Adam. These two optimizers were applied jointly at each training step.

Training was performed using mini-batches of state-value pairs. This enables scaling to large datasets while maintaining the consistency of uncertainty estimates.
\subsubsection{Fitted Value Iteration with GP Targets}
During the fitted value iteration, the GP is trained on targets that are updated iteratively using Bellman backups. At iteration $k$, the target for transition $(s, a, r, s', d)$ is
\[
y^{(k+1)} = r + \gamma (1 - d)\, \mu^{(k)}(s'),
\]
The FVI is first applied using only offline data to obtain a stable initialization of the value function. After each active data collection phase, the FVI is resumed using the augmented dataset, allowing the value information to propagate into the newly explored regions of the state space.
\subsubsection{Role of Uncertainity}
The GP posterior variance is interpreted as an epistemic uncertainty arising from limited data coverage. This uncertainty is used to guide active exploration by prioritizing transitions in which the value estimate is most uncertain. Importantly, uncertainty is never used directly in the control objective; it only influences the data collected.
\begin{table}[H]
\centering
\caption{Sparse Variational Gaussian Process (SVGP) Training Hyperparameters}
\label{tab:gp_hyperparams}
\begin{tabular}{l c}
\hline
\textbf{Hyperparameter} & \textbf{Value} \\
\hline
GP model & Sparse Variational GP (SVGP) \\
Kernel & RBF (Squared Exponential) \& Matern \\
Mean function & Zero / Constant Mean \\
Likelihood & Gaussian \\
Variational inference & ELBO optimization \\
Variational distribution & NaturalVariationalDistribution \\
Variational optimizer & Natural Gradient Descent (NGD) \\
NGD learning rate & 0.01 \\
Kernel hyperparameter optimizer & Adam \\
Adam learning rate & $1 \times 10^{-3}$ \\
Batch size & 2048--4096 \\
Initial GP training epochs & 10-20 \\
Active-phase GP training epochs & 5 \\
Inducing point initialization & MiniBatch K-Means \\
Learn inducing locations & Yes \\
\hline
\end{tabular}
\end{table}
\begin{table}[H]
\centering
\caption{Number of Inducing Points Used in SVGP}
\label{tab:inducing_points}
\begin{tabular}{l c}
\hline
\textbf{Environment} & \textbf{Number of Inducing Points} \\
\hline
Maze2D / AntMaze & 5,000 -- 25,000 \\
HalfCheetah & 5,000--10,000 \\
Hopper & 5,000--10,000 \\
Walker2d & 5,000--10,000 \\
\hline
\end{tabular}
\end{table}

\begin{table}[H]
\centering
\caption{Fitted Value Iteration and active GP update configuration.}
\label{tab:fvi_config}
\begin{tabular}{l c}
\hline
\textbf{Parameter} & \textbf{Value} \\
\hline
Offline FVI iterations & 10--20 \\
Active FVI iterations & 5 \\
Discount factor $\gamma$ & 0.99 \\
Active transitions per batch & 5,000 \\
GP epochs per offline FVI & 3--5 \\
GP epochs per active FVI & 2 \\
\hline
\end{tabular}
\end{table}

\begin{table}[H]
\centering
\caption{Control policy training schedule.}
\label{tab:policy_training}
\begin{tabular}{l c c}
\hline
\textbf{Setting} & \textbf{Navigation} & \textbf{Locomotion} \\
\hline
Offline RL algorithm & TD3+BC / IQL & TD3+BC / IQL \\
Offline training steps & $|\mathcal{D}_{\text{offline}}|$ (TD3+BC) & $|\mathcal{D}_{\text{offline}}|$ \\
 & $2\times|\mathcal{D}_{\text{offline}}|$ (IQL) & $2\times|\mathcal{D}_{\text{offline}}|$ (IQL) \\
Small active update interval & 1k steps & 2k steps \\
Small update policy steps & 5k (TD3+BC) / 10k (IQL) & 10k / 20k \\
Large active update interval & 5k steps & 10k steps \\
Large update policy steps & 25k / 50k & 50k / 100k \\
Batch size & 256 & 256 \\
Network architecture & 2-layer MLP (256 units) & 2-layer MLP (256 units) \\
\hline
\end{tabular}
\end{table}

\subsection{Control Policy Training Details}
We trained the control policies using TD3+BC and IQL, depending on the environment. All optimizer settings and algorithm-specific hyperparameters followed standard configurations reported in prior work and were kept fixed across experiments.

Each policy is first trained offline on the pruned dataset until convergence. During active learning, the policy is periodically retrained using a mixed dataset that combines offline data with newly collected transition data. We employed two retraining schedules. \textbf{Small retraining phases} are triggered frequently and use only the most recent batch of active data, allowing rapid adaptation to newly explored regions. \textbf{Large retraining phases} occur less frequently and use the full mixed dataset, ensuring stability and preventing catastrophic forgetting of previously learned behaviors.

For navigation tasks, we used TD3+BC for the Maze2D environments and IQL for the AntMaze environments. The number of gradient updates during offline training and each active retraining phase is scaled with the size of the dataset and differs between navigation and locomotion tasks, as detailed in Table \ref{tab:policy_training}

\section{Additional Results}
\label{app:additional_results}
% \begin{figure}[t]
%     \centering
%     \includegraphics[
%         width=0.95\textwidth
%     ]{figures/plots/learning_uncertainty_infogain_5x3_panel.png}
%     \caption{
% Learning dynamics in different environments.
% Each row corresponds to a different environment, and columns show
% (\textbf{a}) the learning curve under active data collection,
% (\textbf{b}) the reduction of GP epistemic uncertainty over active steps,
% and (\textbf{c}) the relationship between posterior uncertainty and cumulative information gain.
% }
% \label{fig:learning_curves_panel}
% \end{figure}

\begin{figure}[t]
    \centering
    \includegraphics[
        width=0.85\textwidth
    ]{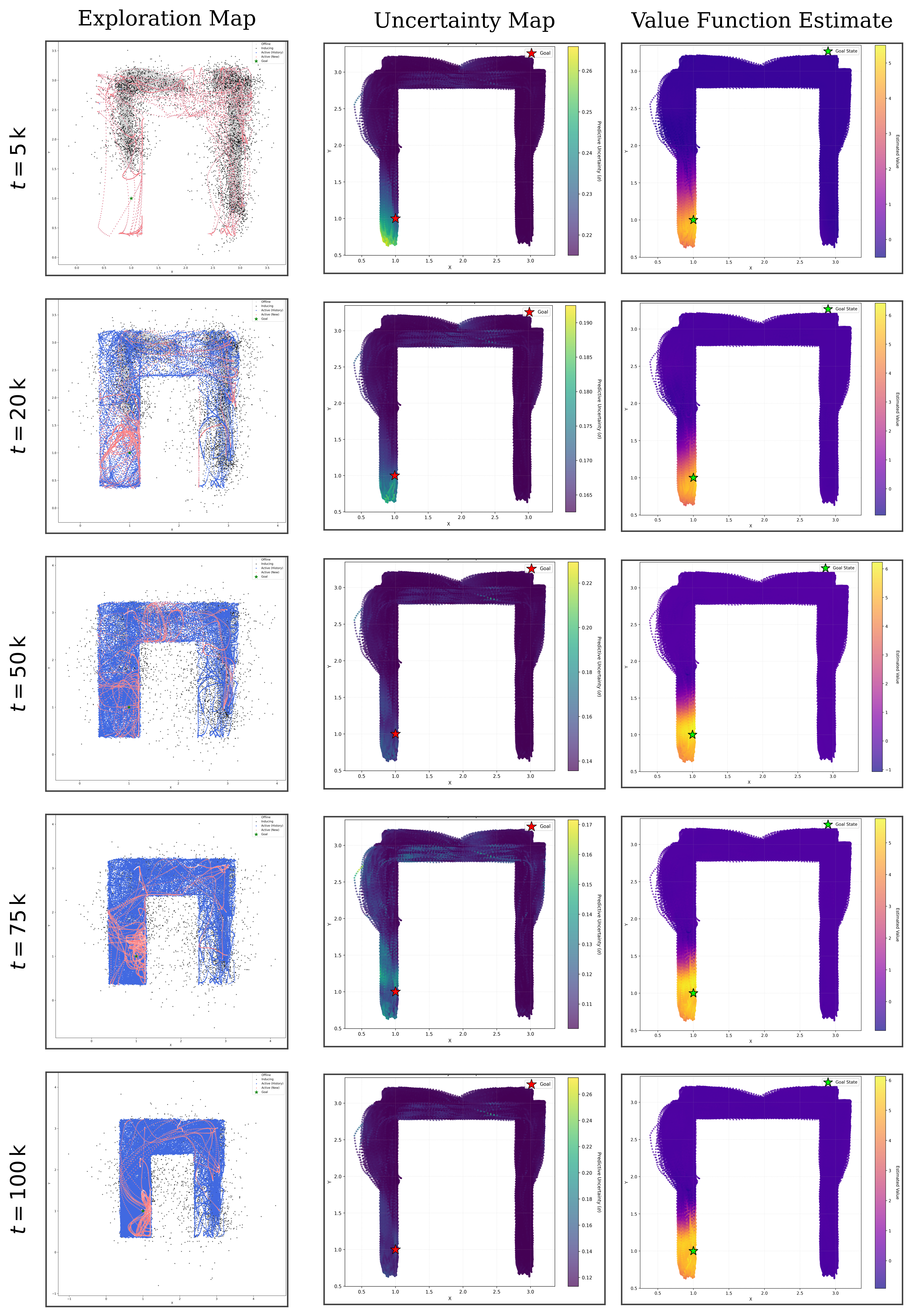}
    \caption{
Evolution of exploration, uncertainty, and value estimates during active learning in \texttt{maze2d-umaze-hard}.
Columns show \textbf{(a)} sampled exploration trajectories, \textbf{(b)} GP epistemic uncertainty $\sigma_T(s)$, and
\textbf{(c)} learned value function estimates $V(s)$.
The rows correspond to increasing active interaction budgets.
As exploration progresses, uncertainty concentrates and decreases in frequently visited regions,
while the value function sharpens around goal-reaching trajectories, illustrating uncertainty-driven
data collection and its effects on policy learning.
}
\label{fig:umaze_panel}
\end{figure}

\begin{figure}[t]
    \centering
    \includegraphics[
        width=0.85\textwidth
    ]{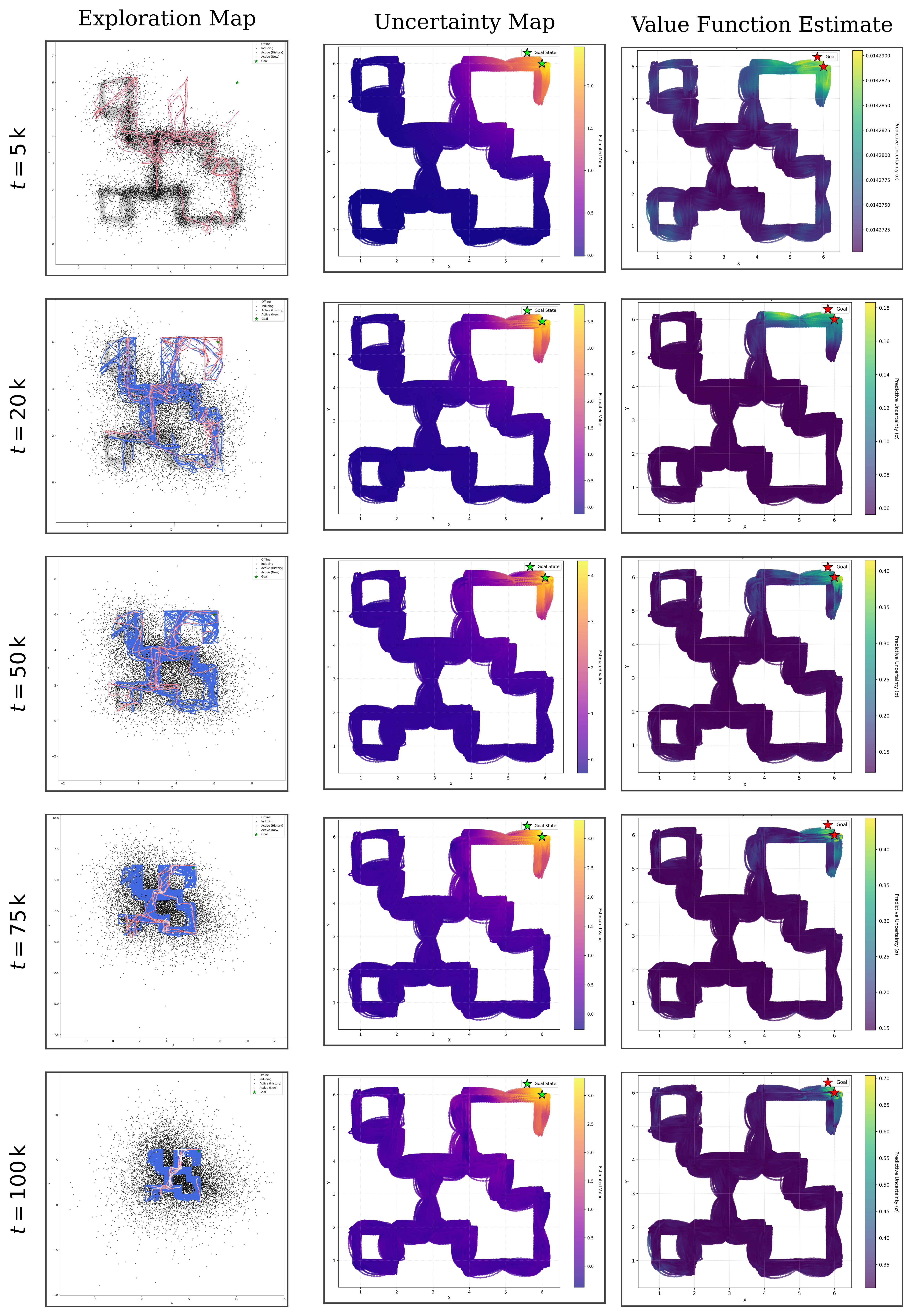}
    \caption{
Evolution of exploration, uncertainty, and value estimates during active learning in \texttt{maze2d-medium-hard}.
Columns show \textbf{(a)} sampled exploration trajectories, \textbf{(b)} GP epistemic uncertainty $\sigma_T(s)$, and
\textbf{(c)} learned value function estimates $V(s)$.
}
    \label{fig:maze2d_medium_panel}
\end{figure}

%%%%%%%%%%%%%%%%%%%%%%%%%%%%%%%%%%%%%%%%%%%%%%%%%%%%%%%%%%%%%%%%%%%%%%%%%%%%%%%
%%%%%%%%%%%%%%%%%%%%%%%%%%%%%%%%%%%%%%%%%%%%%%%%%%%%%%%%%%%%%%%%%%%%%%%%%%%%%%%

\end{document}